\pdfoutput=1 
\documentclass{article} 
\usepackage{iclr2022_conference,times}
\iclrfinalcopy 
\usepackage[utf8]{inputenc} 
\usepackage[T1]{fontenc}    
\usepackage{url}            
\usepackage{booktabs}       
\usepackage{amsfonts}       
\usepackage{nicefrac}       
\usepackage{microtype}      
\usepackage{natbib}
\usepackage{colortbl}
\usepackage{multirow}
\usepackage{mathtools}
\usepackage{tikz}
\usetikzlibrary{fit,calc}
\usetikzlibrary{tikzmark}
\usepackage{url}
\usepackage{graphicx}
\usepackage{amsmath}
\usepackage{enumitem}
\usepackage{mdframed}
\usepackage{bbm}
\usepackage[utf8]{inputenc}
\usepackage[english]{babel}
\usepackage{soul}
\usepackage{amsthm}
\usepackage{algorithm}
\usepackage[noend]{algorithmic}
\usepackage{color}
\usepackage{comment}
\usepackage{wrapfig}
\usepackage{caption}
\usepackage{subcaption}
\usepackage{footmisc}

\usepackage{amsthm}
\usepackage{amsmath,amsfonts,bm}
\makeatletter
\newtheorem*{rep@theorem}{\rep@title}
\newcommand{\newreptheorem}[2]{%
\newenvironment{rep#1}[1]{%
 \def\rep@title{#2 \ref{##1}}%
 \begin{rep@theorem}}%
 {\end{rep@theorem}}}
\makeatother

\newtheorem{definition}{Definition}[section]
\newtheorem{theorem}{Theorem}
\newtheorem{assumption}{Assumption}
\newreptheorem{theorem}{Theorem}
\newtheorem{lemma}{Lemma}

\newcommand{\RNum}[1]{\uppercase\expandafter{\romannumeral #1\relax}}

\usepackage{mathtools}

\DeclareMathOperator*{\argmax}{\arg\!\max}
\DeclareMathOperator*{\topk}{top_k}

\newcommand{\xhdr}[1]{{\noindent\bfseries #1}.}
\newcommand{\cut}[1]{}

\newcommand{\removelatexerror}{\let\@latex@error\@gobble}

\def\eqref#1{Eq.~\ref{#1}}

\def\floor#1{\lfloor #1 \rfloor}
\def\1{\bm{1}}

\DeclareMathAlphabet{\mathsfit}{\encodingdefault}{\sfdefault}{m}{sl}
\SetMathAlphabet{\mathsfit}{bold}{\encodingdefault}{\sfdefault}{bx}{n}

\def\gF{{\mathcal{F}}}

\def\gX{{\mathcal{X}}}

\def\sP{{\mathbb{P}}}

\usepackage{hyperref}
\usepackage{url}

\newcounter{ggiCounter}

\colorlet{pink}{red!40}
\colorlet{blue}{cyan!60}
\setstcolor{red}
\definecolor{salmon}{HTML}{FF8C94}
\definecolor{mydarkblue}{rgb}{0,0.08,0.45}
\definecolor{grannysmithapple}{rgb}{0.79, 0.92, 0.77}
\definecolor{darkgrannysmithapple}{rgb}{0.79, 0.99, 0.77}
\definecolor{g1}{HTML}{8FC771}
\definecolor{g2}{HTML}{B4D375}
\definecolor{g3}{HTML}{DADF79}
\definecolor{g4}{HTML}{FFEB7D}

\newcommand{\algoname}{\textsc{Virtual+}}
\newcommand{\setvalue}{$\mathbb{V}$}
\newcommand{\setvaluemath}{\mathbb{V}}
\newtheorem*{theo}{Online Threat Model}
\newenvironment{ftheo}
  {\begin{mdframed}\begin{theo}}
  {\end{theo}\end{mdframed}}
  
\title{Online Adversarial Attacks}

\author{Andjela Mladenovic$^*$ \\
  Mila, Université de Montréal  \\
  \And
  Avishek Joey Bose\thanks{Equal Contribution. Corresponding authors: \{\texttt{joey.bose,andjela.mladenovic\}@mila.quebec}} \\
  Mila, McGill University 
  \And
  \And
  Hugo Berard$^{*}$\thanks{Work done while an intern at Meta AI Research}\\
  Mila, Université de Montréal \\
    \And
  William L. Hamilton$^{\ddagger}$\\
  Mila, McGill University \\
  
    \And
  Simon Lacoste-Julien$^{\ddagger}$ \\
  Mila, Université de Montréal
  \And
  Pascal Vincent\thanks{Canada CIFAR AI Chair} \\
  Mila, Université de Montréal\\
  Meta AI Research\\[-6mm]

  \And
  Gauthier Gidel$^{\ddagger}$ \\
  Mila, Université de Montréal 
}

\begin{document}

\maketitle
\vspace{-10pt}
\begin{abstract}
Adversarial attacks expose important vulnerabilities of deep learning models, yet little attention has been paid to settings where data arrives as a stream. In this paper, we formalize the online adversarial attack problem, emphasizing two key elements found in real-world use-cases: attackers must operate under partial knowledge of the target model, and the decisions made by the attacker are irrevocable since they operate on a transient data stream. We first rigorously analyze a deterministic variant of the online threat model by drawing parallels to the well-studied $k$-secretary problem in theoretical computer science and propose \algoname, a simple yet practical online algorithm. Our main theoretical result shows \algoname \ yields provably the best competitive ratio over all single-threshold algorithms for $k<5$---extending the previous analysis of the $k$-secretary problem. We also introduce the \textit{stochastic $k$-secretary}---effectively reducing online blackbox transfer attacks to a $k$-secretary problem under noise---and prove theoretical bounds on the performance of \algoname \ adapted to this setting. Finally, we complement our theoretical results by conducting experiments on MNIST, CIFAR-10, and Imagenet classifiers, revealing the necessity of online algorithms in achieving near-optimal performance and also the rich interplay between attack strategies and online attack selection, enabling simple strategies like FGSM to outperform stronger adversaries.
\vspace{-5pt}
\end{abstract}

\section{Introduction}
\looseness=-1
\vspace{-10pt}
In adversarial attacks, an attacker seeks to maliciously disrupt the performance of deep learning systems by adding small but often imperceptible noise to otherwise clean data~\citep{szegedy2013intriguing,goodfellow2014explaining}. Critical to the study of adversarial attacks is specifying the threat model~\cite{akhtar2018threat}, which outlines the adversarial capabilities of an attacker and the level of information available in crafting attacks. Canonical examples include the \textit{whitebox} threat model~\cite{madry2017towards}, where the attacker has complete access, and the less permissive \textit{blackbox} threat model where an attacker only has partial information, like the ability to query the target model \citep{chen2017zoo,ilyas2017query,papernot2016transferability}. 

Previously studied threat models (e.g., whitebox and blackbox) implicitly assume a static setting that permits full access to instances in a target dataset at all times~\citep{tramer2017ensemble}. However, such an assumption is unrealistic in many real-world systems. Countless real-world applications involve streaming data that arrive in an online fashion (e.g., financial markets or real-time sensor networks). Understanding the feasibility of adversarial attacks in this {\em online} setting is an essential question. 

As a motivating example, consider the case where the adversary launches a man-in-the-middle attack depicted in Fig.~\ref{fig:man_in_the_middle}. Here, data is streamed between two endpoints---i.e., from sensors on an autonomous car to the actual control system. 
An adversary, in this example, would intercept the sensor data, potentially perturb it, and then send it to the controller. Unlike classical adversarial attacks, such a scenario presents two key challenges that are representative of all online settings. 

\begin{wrapfigure}{o}{0.4\textwidth}
  \vspace{-12pt}
    \includegraphics[width=0.38\textwidth]{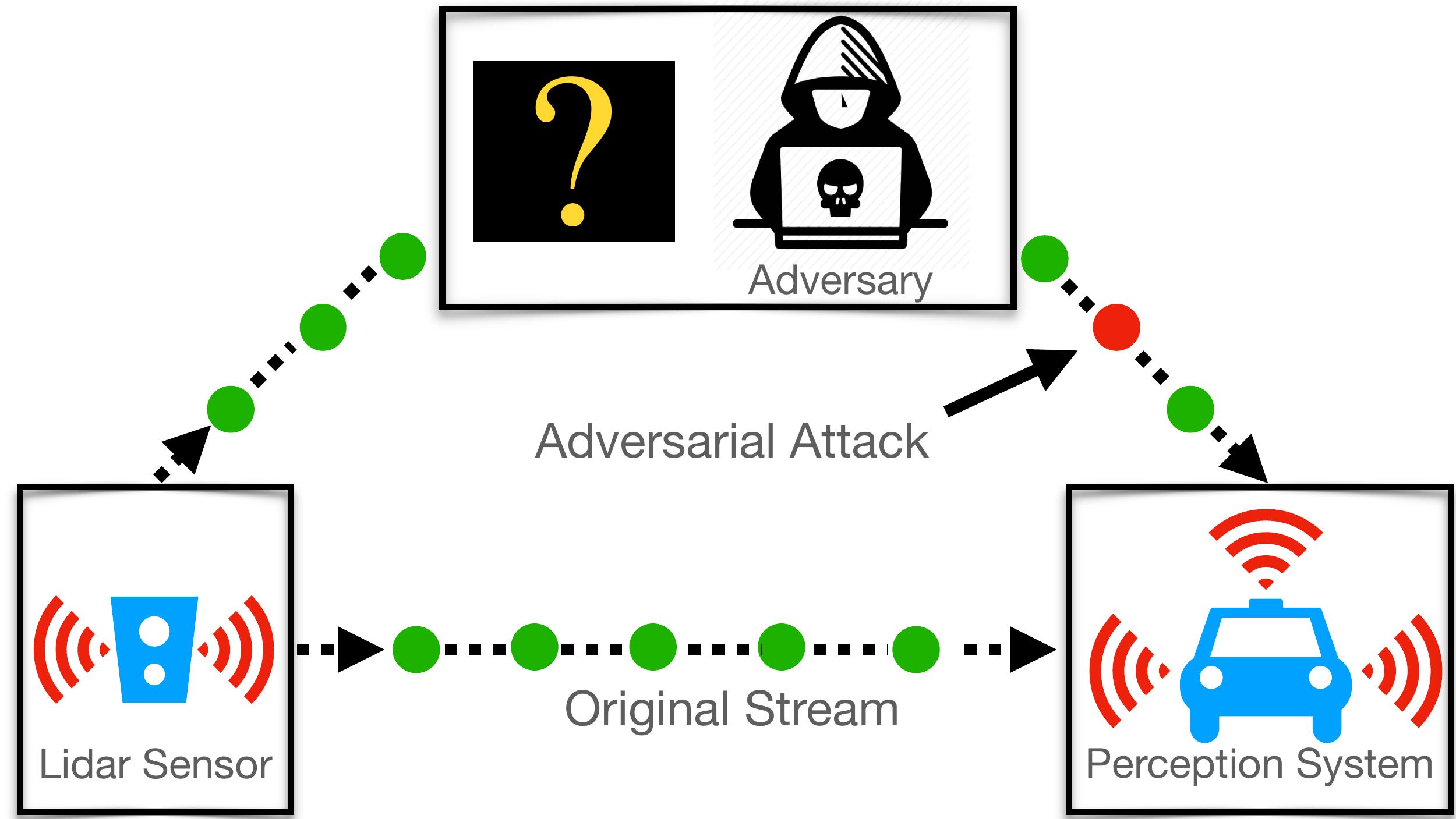}
  \vspace{-2pt}
\caption{Man-in-the-Middle Attack.}
\label{fig:man_in_the_middle}
\end{wrapfigure}

\begin{enumerate}[noitemsep,topsep=0pt,parsep=0pt,partopsep=0pt, leftmargin=*] 
    \item \textbf{Transiency:} At every time step, the attacker makes an irrevocable decision on whether to attack, and if she fails, or opts not to attack, then that datapoint is no longer available for further attacks.
    \item \textbf{Online Attack Budget:} The adversary---to remain anonymous from stateful defenses ---is restricted to a small selection budget and must optimally balance a passive exploration phase before selecting high-value items in the data stream (e.g. easiest to attack) to submit an attack on.
\end{enumerate}

To the best of our knowledge, the only existing approaches that craft adversarial examples on streaming data \citep{gong2019real,lin2017tactics,sun2020stealthy} require multiple passes through a data stream and thus cannot be applied in a realistic online setting where an adversary is forced into irrevocable decisions. Moreover, these approaches do not come with theoretical guarantees. Consequently, assessing the practicality of adversarial attacks---to better expose risks---in a truly online setting is still an open problem, and the focus of this paper.

\xhdr{Main Contributions} 
\looseness=-1
We formalize the online threat model to study adversarial attacks on streaming data. In our online threat model, the adversary must execute $k$ successful attacks within $n$  streamed data points, where $k \ll n$. As a starting point for our analysis, we study the deterministic online threat model in which the actual value of an input---i.e., the likelihood of a successful attack---is revealed along with the input. Our first insight elucidates that such a threat model, modulo the attack strategy, equates to the $k$-secretary problem known in the field of optimal stopping theory \cite{dynkin1963optimum,kleinberg2005multiple}, allowing for the application of established online algorithms for picking optimal data points to attack. We then propose a novel online algorithm \algoname\ that is both practical, simple to implement for any pair $(k,n)$, and requires no additional hyperparameters.

Besides, motivated by attacking blackbox target models, we also introduce a modified secretary problem dubbed the \textit{stochastic $k$-secretary problem}, which assumes the values an attacker observes are stochastic estimates of the actual value. We prove theoretical bounds on the competitive ratio---under mild feasibility assumptions---for \algoname in this setting. Guided by our theoretical results, we conduct a suite of experiments on both toy and standard datasets and classifiers (i.e., MNIST, CIFAR-10, and Imagenet). Our empirical investigations reveal two counter-intuitive phenomena that are unique to the online blackbox transfer attack setting: 1.) In certain cases attacking robust models may in fact be easier than non-robust models based on the distribution of values observed by an online algorithm. 2.) Simple attackers like FGSM can seemingly achieve higher online attack transfer rates than stronger PGD-attackers when paired with an online algorithm, demonstrating the importance of carefully selecting which data points to attack. We summarize our key contributions:

\begin{itemize}[noitemsep,topsep=0pt,parsep=0pt,partopsep=0pt,label={\large\textbullet},leftmargin=*]
\item We formalize the online adversarial attack threat model as an online decision problem and rigorously connect it to a generalization of the k-secretary problem.
\item We introduce and analyze \algoname, an extension of \textsc{Virtual} for the $k$-secretary problem yielding a significant practical improvement ($60\%$).  

We then provide, via novel techniques, a tractable formula for its competitive ratio, partially answering one of \citet{albers2020new}'s open questions (see footnote~\footref{foot:closed_form}) and achieving a new state-of-the-art competitive ratio for $k<5$.
\item We propose Alg.~\ref{alg:online_adv_attack} that leverages (secretary) online algorithms to perform efficient online adversarial attacks. We compare different online algorithms including \algoname\ on MNIST, CIFAR-10, and Imagenet in the challenging Non-Interactive BlackBox transfer (NoBox) setting.
\end{itemize}

\section{Background and Preliminaries}
\looseness=-1
\vspace{-10pt}

\xhdr{Classical Adversarial Attack Setup}
We are interested in constructing adversarial examples against some fixed 
target classifier $f_{t}: \mathcal{X} \to \mathcal{Y}$ which consumes input data points $x \in \mathcal{X}$ and labels them with a class label $y \in \mathcal{Y}$. The goal of an adversarial attack is then to produce an adversarial example $x' \in \mathcal{X}$, such that $f_t(x') \neq y$, and where the distance $d(x,x') \leq \gamma$. Then, equipped with a loss $\ell$ used to evaluate $f_{t}$, an attack is said to be optimal if \citep{carlini2017magnet, madry2017towards}, 
\begin{equation}\label{eq:opt_attack}
    \textstyle x' \in \argmax_{x'\in \gX}\ell(f_{t}(x'),y) \,, \;\; \text{s.t.} \;\; d(x,x') \leq \gamma \,.
\end{equation}
Note that the formulation above makes no assumptions about access and resource restrictions imposed upon the adversary. Indeed, if the parameters of $f_t$ are readily available, we arrive at the familiar whitebox setting, and problem in \eqref{eq:opt_attack} is solved by following the gradient $\nabla_x f_{t}$ that maximizes $\ell$. 

\xhdr{k-Secretary Problem}
The secretary problem is a well-known problem in theoretical computer science \cite{dynkin1963optimum,ferguson1989solved}. Suppose that we are tasked with hiring a secretary from a randomly ordered set of $n$ potential candidates to select the secretary with maximum value. The secretaries are interviewed sequentially and reveal their actual value on arrival. Thus, the decision to accept or reject a secretary must be made immediately, irrevocably, and without knowledge of future candidates. While there exist many generalizations of this problem, in this work, we consider one of the most canonical generalizations known as the \textit{$k$-secretary problem} \cite{kleinberg2005multiple}. Here, instead of choosing the best secretary, we are tasked with choosing $k$ candidates to maximize the expected sum of values. Typically, online algorithms that attempt to solve secretary problems are evaluated using the competitive ratio, which is the value of the objective achieved by an online algorithm compared to an optimal value of the objective that is achieved by an ideal ``offline algorithm,” i.e., an algorithm with access to the entire candidate set. Formally, an online algorithm $\mathcal{A}$ that selects a subset of items $S_{\mathcal{A}}$  is said to be $C$-competitive to the optimal algorithm $\textrm{OPT}$ which greedily selects a subset of items $S^*$ while having full knowledge of all $n$ items, if asymptotically in $n$
\begin{equation}
    \mathbb{E}_{\pi\sim \mathcal{S}_n}[\setvaluemath(S_{\mathcal{A}})] \geq  (C + o(1)) \setvaluemath(S^*) \,,
    \label{eq:comp_ration}
\end{equation}
where \setvalue\ is a set-value function that determines the sum utility of each algorithm's selection, and the expectations are over permutations sampled from the symmetric group of $n$ elements, $\mathcal{S}_n$, acting on the data. In \S\ref{stochastic_k_secretary}, we shall further generalize the $k$-secretary problem to its stochastic variant where the online algorithm is no longer privy to the actual values but must instead choose under uncertainty.

\section{Online Adversarial Attacks}
\looseness=-1
\vspace{-10pt}
\label{online_adversarial_attacks}
\looseness=-1
Motivated by our more realistic threat model, we now consider a novel adversarial attack setting where the data is no longer static but arrives in an online fashion.

\subsection{Adversarial Attacks as Secretary Problems}
\looseness=-1
\vspace{-5pt}
\label{adv_sec_problem}
The defining feature of the online threat model---in addition to streaming data and the fact that we may not have access to the target model $f_t$---is the online attack budget constraint.
Choosing when to attack under a fixed budget in the online setting can be related to a secretary problem. We formalize this online adversarial attack problem in the boxed online threat model below.

In the online threat model we are given a data stream $\mathcal{D}=\{(x_1,y_1),\ldots,(x_n,y_n)\}$ of $n$ samples ordered by their time of arrival. In order to craft an attack against the target model $f_t$, the adversary selects, using its online algorithm $\mathcal{A}$, a subset $S_{\mathcal{A}} \subset \mathcal{D}$ of items to maximize:
\begin{equation}\label{eq:asp}
   \setvaluemath(S_\mathcal{A}) \! := \!\!\! \sum_{ (x,y) \in S_\mathcal{A}} \ell(f_{t}(\textsc{Att}(x)),y) \ \text{ s.t. } 
   |S_A| \leq k,\hspace{-1mm} 
\end{equation}
\cut{
\begin{equation}\label{eq:asp}
   \setvaluemath(S_\mathcal{A}) \! := \!\!\! \sum_{ i \in S_\mathcal{A}}v_i \text{ s.t. }  v_i = \ell(f_{t}(x_i'),y_i) \text{ and }
   |S_A| \leq k,\hspace{-1mm} 
\end{equation}
}
where $\textsc{Att}(x)$ denotes an attack on $x$ crafted by a \emph{fixed} attack method $\textsc{Att}$ that might or might not depend on $f_t$. From now on we define $x_i'=\textsc{Att}(x_i)$. 
Intuitively, the adversary chooses $k$ instances that are the ``easiest" to attack, i.e. samples with the highest value. Note that selecting an instance to attack does not guarantee a successful attack.
Indeed, a successful attack vector may not exist if the perturbation budget $\gamma$ is too small. \cut{even though the value is maximized.} However, stating the adversarial goal as maximizing the value of $S_{\mathcal{A}}$ leads to the measurable objective of calculating the ratio of successful attacks in $S_{\mathcal{A}}$ versus $S^*$.

If the adversary knows the true value of a datapoint then the online attack problem reduces to the original $k$-secretary. On the other hand, the adversary might not have access to $f_t$, and instead, the adversary's value function may be an estimate of the true value---e.g., \ the loss of a surrogate classifier, and the adversary must make selection decisions in the face of uncertainty.\cut{ yielding a stochastic generalization of the $k$-secretary problem.} The theory developed in this paper will tackle both the case where values $v_i:=\ell(f_t(x_i'),y_i)$ for $i \in \{1,\ldots,n\} :=[n]$ are known (\S\ref{virtual_plus}), as well as the richer stochastic setting with only estimates of $v_i\,,\, i \in [n]$ (\S\ref{stochastic_k_secretary}).

\xhdr{Practicality of the Online Threat Model} It is tempting to consider whether in practice the adversary should forego the online attack budget and instead attack every instance. However, such a strategy poses several critical problems when operating in real-world online attack scenarios. Chiefly, attacking any instance in $\mathcal{D}$ incurs a non-trivial risk that the adversary is detected by a defense mechanism. Indeed, when faced with stateful defense strategies (e.g. \cite{chen2020stateful}), every additional attacked instance further increases the risk of being detected and rendering future attacks impotent. Moreover, attacking every instance may be infeasible computationally for large $n$ or impractical based on other real-world constraints. Generally speaking, as conventional adversarial attacks operate by restricting the perturbation to a fraction of the maximum possible change (e.g., $\ell_{\infty}$-attacks), online attacks analogously restrict the time window to a fraction of possible instances to attack. Similarly, knowledge of $n$ is also a factor that the adversary can easily control in practice. For example, in the autonomous control system example, the adversary can choose to be active for a short interval---e.g., when the autonomous car is at a particular geospatial location---and thus set the value for $n$.
\begin{ftheo}
The online threat model relies on the following key definitions:
\begin{itemize}[leftmargin=*, itemsep=1pt, topsep=1pt, parsep=1pt]
\item 
\textbf{The target model $f_t$}. The adversarial goal is to attack some target model $f_t : \mathcal{X} \rightarrow \mathcal{Y}$, through adversarial examples that respect a chosen distance function, $d$, with tolerance $\gamma$. 

\item
\textbf{The data stream $\mathcal{D}$}. The data stream $\mathcal{D}$ contains the $n$ examples $(x_i,y_i)$ ordered by their time of arrival. At any timestep $i$, the adversary receives the corresponding item in $\mathcal{D}$ and must decide whether to execute an attack or forever forego the chance to attack this item.

\item
\textbf{Online attack budget $k$}. The adversary is limited to a maximum of $k$ attempts to craft attacks within the online setting, thus imposing that each attack is on a unique item in $\mathcal{D}$.

\item
\textbf{A value function $\mathcal{V}$}. Each item in the dataset is assigned a value on arrival by the value function $\mathcal{V}: \mathcal{X} \times \mathcal{Y} \rightarrow \mathbb{R}_+$ which represents the utility of selecting the item to craft an attack. This can be the likelihood of a successful attack under $f_t$ (true value) or a stochastic estimate of the incurred loss given by a surrogate model $f_s \approx f_t$.
\end{itemize}

The online threat model corresponds to the setting where the adversary seeks to craft adversarial attacks (i) against a target model $f_t \in \gF$, (ii) by observing items in $\mathcal{D}$ that arrive online, (iii) and choosing $k$ optimal items to attack by relying on (iv) an available value function $\mathcal{V}$. The adversary's objective is then to use its value function towards selecting items in $\mathcal{D}$ that maximize the sum total value of selections \setvalue\ (Eq.~\ref{eq:asp}).
\end{ftheo}

\subsection{\algoname\ for Adversarial Secretary Problems}
\looseness=-1
\vspace{-5pt}
\label{virtual_plus}
Let us first consider the deterministic variant of the online threat model, where the true value is known on arrival. For example consider the value function $\mathcal{V}(x_i,y_i) = \ell(f_{t}(x'_i),y_i) = v_i$ i.e. the loss resulting from the adversary corrupting incoming data $x_i$ into $x'_i$. Under a fixed attack strategy, the selection of high-value items from $\mathcal{D}$ is exactly the original $k$-secretary problem and thus the adversary may employ any $\mathcal{A}$ that solves the original $k$-secretary problem.

Well-known single threshold-based algorithms that solve the $k$-secretary problem include the \textsc{Virtual}, \textsc{Optimistic} \cite{babaioff2007knapsack} and the recent \textsc{Single-Ref} algorithm \cite{albers2020new}. In a nutshell, these online algorithm consists of two phases---a \textit{sampling phase} followed by a \textit{selection phase}---and an optimal stopping point $t$ (threshold) that is used by the algorithm to transition between the phases. In the sampling phase, the algorithms passively observe all data points up to a pre-specified threshold $t$. Note that $t$ itself is algorithm-specific and can be chosen by solving a separate optimization problem. Additionally, each algorithm also maintains a sorted reference list $R$ containing the top-$k$ elements. Each algorithm then executes the selection phase through comparisons of incoming items to those in $R$ and possibly updating $R$ itself in the process (see~\S\ref{appendix:classical_online_algorithms}).

Indeed, the simple structure of both the \textsc{Virtual} and \textsc{Optimistic} algorithms---e.g., having few hyperparameters and not requiring the algorithm to involve Linear Program's for varying values of $n$ and $k$---in addition to being $(1/e)$-competitive (optimal for $k=1$) make them suitable candidates for solving \eqref{eq:asp}. However, 
the competitive ratio of both algorithms in the small $k$ regime---but not $k=1$---has shown to be sub-optimal with \textsc{Single-Ref} provably yielding larger competitive ratios at the cost of an additional hyperparameter selected via combinatorial optimization when $n \to \infty$. 

We now present a novel online algorithm, \algoname, that retains the simple structure of \textsc{Virtual} and \textsc{Optimistic}, with no extra hyperparameters, but leads to a new state-of-the-art competitive ratio for $k<5$. Our key insight is derived from re-examining the selection condition in the \textsc{Virtual} algorithm and noticing that it is overly conservative and can be simplified. The \algoname\ algorithm is presented in Algorithm 1, where the removed condition in \textsc{Virtual} (L2-3) is \st{in pink strikethrough}. Concretely, the condition that is used by \textsc{Virtual} but \emph{not} by \algoname\ updates $R$ during the selection phase without actually picking the item as part of $S_{\mathcal{A}}$. Essentially, this condition is theoretically convenient and leads to a simpler analysis by ensuring that the \textsc{Virtual} algorithm never exceeds $k$ selections in $S_{\mathcal{A}}$. \algoname\ removes this conservative $R$ update criteria in favor of a simple to implement condition, $|S_{\mathcal{A}}| \leq k$ line 4 {\color{salmon}(in pink)}. Furthermore, the new selection rule also retains the simplicity of \textsc{Virtual} leading to a painless application to online attack problems.

\begin{minipage}[t]{.49\textwidth}
\vspace{-15pt}
\begin{algorithm}[H]
\small
\textbf{Inputs:} $t\in[k\dots n-k]$, $R = \emptyset$, $S_{\mathcal{A}} = \emptyset$
\newline
\textbf{Sampling phase:} Observe the first $t$ data points and construct a sorted list $R$ with the indices of the top $k$ data points seen. The method $\texttt{sort}$ ensures: $ \mathcal{V}(R[1]) \geq \mathcal{V}(R[2]) \dots \geq \mathcal{V}(R[k]).$
\newline
\textbf{Selection phase}:{\color{salmon} \{//\textsc{Virt}+ removes L2-3 and adds L4 \}} \hspace{-4.0cm}
\begin{algorithmic}[1]
\FOR{$i:=t+1$ to $n$ }
    \IF{$\mathcal{V}(i) \geq \mathcal{V}(R[k])$ and $R[k] > t$}
        \STATE $R$ = $\texttt{sort}(R \cup \{i\} \setminus \{R[k]\})$
        \hfill  
        \tikzmark{start}
        \tikzmark{stop}
        \begin{tikzpicture}[remember picture, overlay]
        \draw[salmon,thick] ([xshift=-155pt,yshift=13pt]pic cs:start) -- ([xshift=-20pt, yshift=12pt]pic cs:stop);        \draw[salmon,thick] ([xshift=-155pt,yshift=2pt]pic cs:start) -- ([xshift=-20pt, yshift=2pt]pic cs:stop);
         \draw[salmon,thick] ([xshift=-155pt,yshift=-8pt]pic cs:start) -- ([xshift=-142pt, yshift=-8pt]pic cs:stop);
        \end{tikzpicture}
    \tikzmark{start1}
    \tikzmark{stop2}
    \begin{tikzpicture}[remember picture, overlay]
    \end{tikzpicture}
    \ELSIF {$\mathcal{V}(i) \geq \mathcal{V}(R[k])$ {\color{salmon} and $ |S_{\mathcal{A}}| \leq  k$} }
    \STATE $R$ = $\texttt{sort}(R \cup \{i\} \setminus \{R[k]\})$ \hfill\COMMENT{// Update $R$}
    \STATE $S_{\mathcal{A}} = S_{\mathcal{A}} \cup \{i \}$ \hfill\COMMENT{// Select element $i$}
    \ENDIF
\ENDFOR
\end{algorithmic}
 \caption{\small \textsc{Virtual} and {\color{salmon}\algoname\,} }
 \label{alg:virtual_plus}
\end{algorithm}
\end{minipage}
\hfill
\begin{minipage}[t]{.48\textwidth}
\begin{figure}[H]
 \vspace{-15pt}
    \includegraphics[width=1.01\linewidth]{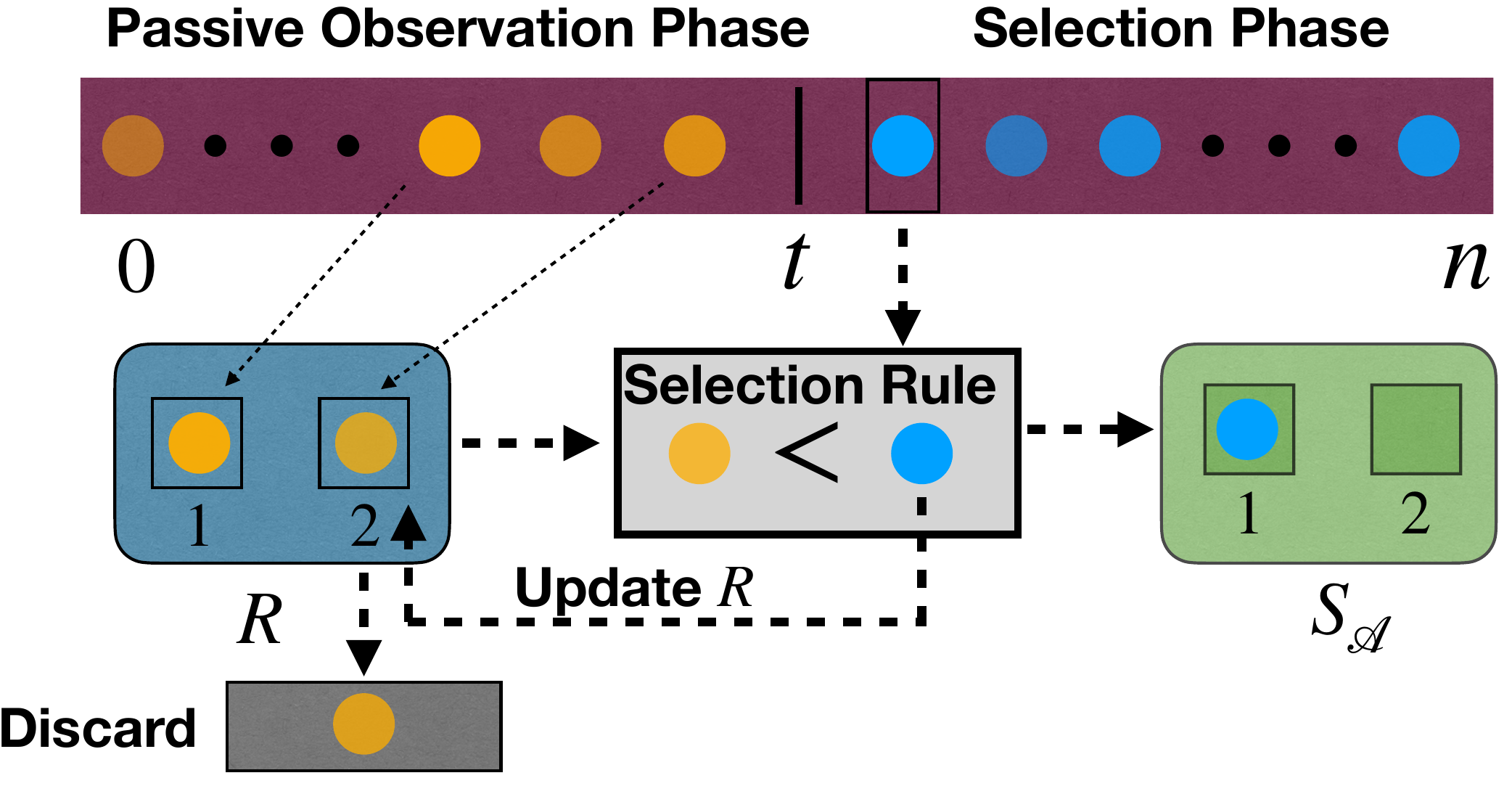}
    \vspace{-10pt}
    \caption{\algoname \ observes $v_i$ (or estimates) and maintains $R$ during the sampling phase. Items are then picked into $S_{\mathcal{A}}$, after threshold $t$, \cut{via comparisons to $R$.}}
    \label{fig:stochastic_secretary}
\end{figure}
\end{minipage}

\cut{By design, \algoname\ also does not exceed $k$ selections and works for any combination of $k$ and $n$.}
\xhdr{Competitive ratio of \algoname}
What appears to be a minor modification in \textsc{Virtual+} compared to \textsc{Virtual} leads to a significantly more involved analysis but a larger competitive ratio. In Theorem 1, we derive the analytic expression that is a tight lower bound for the competitive ratio of \algoname\ for \emph{general}-$k$. We see that \algoname\ provably improves in competitive ratio for $k<5$ over both \textsc{Virtual}, \textsc{Optimistic}, and in particular the previous best single threshold algorithm, \textsc{Single-Ref}.

\begin{theorem}
The competitive ratio of \algoname\ for $k \geq 2$ with threshold $t_k = \alpha n$ can asymptotically be lower bounded by the following concave optimization problem,
\begin{equation}
    C_k \geq  \max_{\alpha \in [0,1]} f(\alpha) :=  {\alpha}^k \sum_{m = 0}^{k - 1} a_m \ln^m (\alpha)- \alpha a_0
    \quad where  \quad
    a_m := \big(\tfrac{k^k}{(k-1)^{k-m}} - k^m\big)\frac{(-1)^{m+1}}{m!}
    \, .
\end{equation}
Particularly, we get $C_2\geq0.427, C_3\geq .457, C_4\geq.4769$ outperforming~\citet{albers2020new}.
\label{thm:general_k_theorem1}
\end{theorem}

\xhdr{Connection to Prior Work}
\label{connection_to_prior_work}
The full proof for Theorem~\ref{thm:general_k_theorem1} can be found in~\S\ref{app:general_k_proof} along with a simple but illustrative proof for $k=2$ in~\S\ref{appendix:virtual_plus_proof_k_2}.
Theorem~\ref{thm:general_k_theorem1} gives a tractable way to compute the competitive ratio of \algoname\ for any $k$, that improve the previous state-of-the-art~\citep{albers2020new} in terms of single threshold $k$-secretary algorithms for $k<5$ and $k>100$.\footnote{\label{foot:closed_form}\citet{albers2020new} only provide competitive ratios of \textsc{Single-Ref} for $k \leq 100$ and conclude that ``a closed formula for the competitive ratio for any value
of $k$ is one direction of future work''. We partially answer this open question by expressing \algoname's optimal threshold $t_k$ as the solution of a uni-dimensional concave optimization problem. In Table~\ref{tab:C_k}, we provide this threshold for a wide range of $k \geq 100$.} 
However, it is also important to contextualize \algoname\ against recent theoretical advances in this space. 
Most prominently, \citet{buchbinder2014secretary} proved that the $k$-secretary problem can be solved {\em optimally} (in terms of competitive ratio) using linear programs (LPs), {\em assuming a fixed length of $n$}. But these optimal algorithms are typically not feasible in practice.
Critically, they require individually tuning multiple thresholds by solving a separate LP with $\Omega(n k^2)$ parameters for each length of the data stream $n$, and the number of constraints grows to infinity as  $n\rightarrow\infty$. \citet{chan2014revealing} showed that optimal algorithms with $k^2$ thresholds could be obtained using infinite LPs and derived an optimal algorithm for $k=2$ Nevertheless they require a large number of parameters and the scalability of infinite LPs for $k>2$ remains uncertain. 
In this work, we focus on practical methods with a \emph{single} threshold (i.e., with $O(1)$ parameters, e.g. Algorithm~\ref{alg:virtual_plus}) that do not require involved computations that grow with $n$.

\xhdr{Open Questions for Single Threshold Secretary Algorithms} \citet{albers2020new} proposed new non-asymptotic results on the $k$-secretary problem that outperform asymptotically optimal algorithms---opening a new range of open questions for the $k$-secretary problem. While this problem is considered solved when working with probabilistic algorithms\footnote{At each timestep a deterministic algorithm chooses a candidate according to a deterministic rule depending on some parameters (usually a threshold and potentially a rank to compare with). A probabilistic algorithm choose to accept a candidate according to $q_{i,j,l}$ the probability of accepting the candidate in $i$-th position as the $j^{th}$ accepted candidate given that the candidate is the $l$-th best candidate among the $i$ first candidates $(i \in [n],\, j,l \in [K]$.) See~\citep{buchbinder2014secretary} for more details on probabilistic secretary algorithms. } with $\Theta(nK^2)$ parameters~\citep{buchbinder2014secretary}, finding optimal non-asymptotic single-threshold ($O(1)$ parameters) algorithms is still an open question. As a step towards answering this question, our work proposes a practical algorithm that improves upon~\citet{albers2020new} for $k=2,\ldots,4$ with an optimal threshold that can be computed easily as it has a closed form.   
\vspace{-10pt}
\section{Stochastic Secretary Problem}
\vspace{-10pt}
\looseness=-1
\label{stochastic_k_secretary}
\cut{
\begin{figure}
    \centering
    \includegraphics[width=0.48\linewidth]{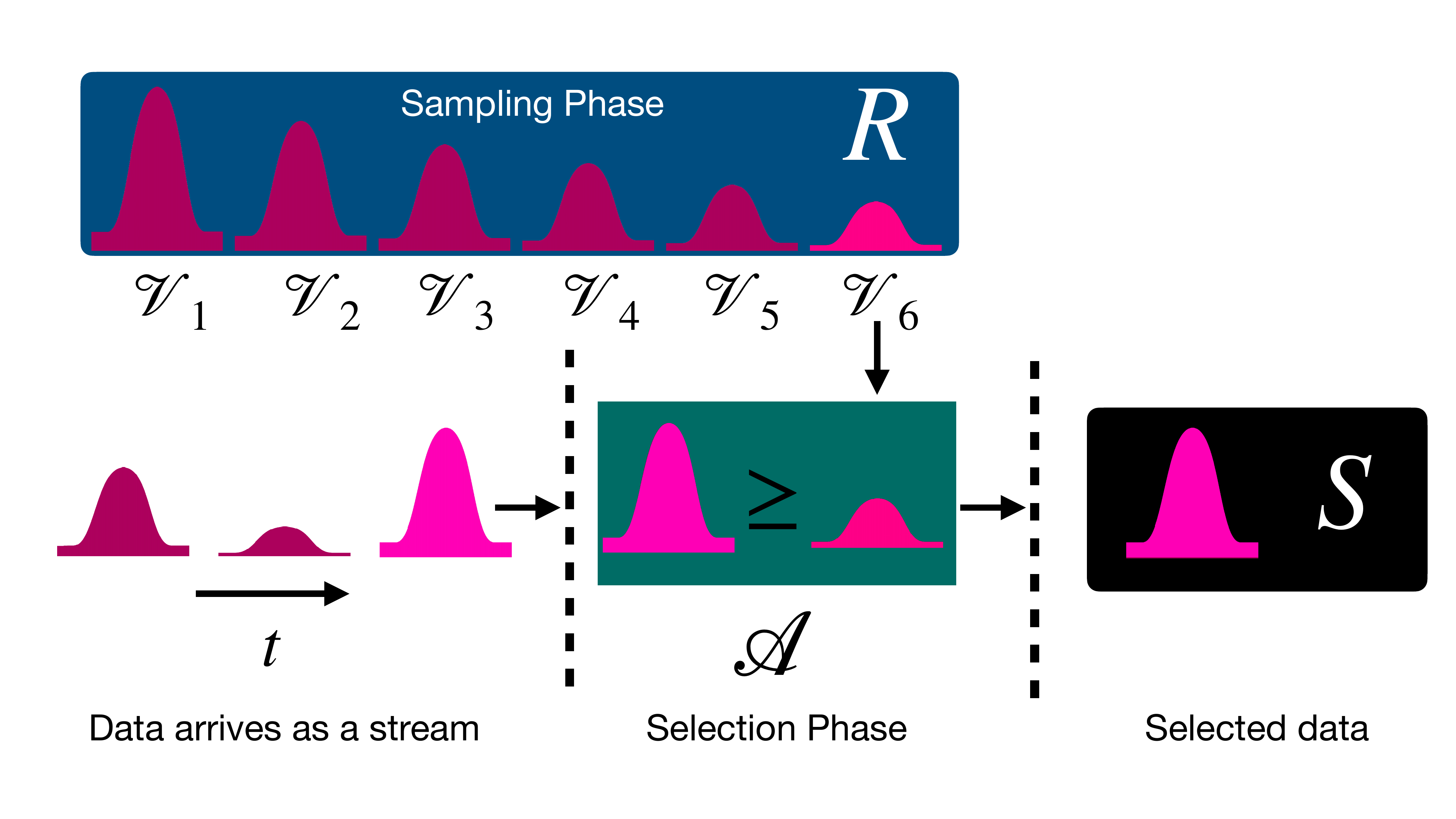}
    \includegraphics[width=0.48\linewidth]{Figures/Virtual+_algo.pdf}
    \caption{Each online algorithm, $\mathcal{A}$, observes estimates of $v_i$ and maintains a reference list $R$ during the sampling phase. Items are then picked into $S_{\mathcal{A}}$, after threshold $t$, via comparisons to $R$.}
    \label{fig:stochastic_secretary}
\end{figure}
}
In practice, online adversaries are unlikely to have access to the target model $f_t$.
Instead, it is reasonable to assume that they have partial knowledge. 

Following \cite{papernot2016practical,bose2020adversarial} we focus on modeling that partial knowledge by equipping the adversary with a surrogate model or representative classifier $f_s$. Using $f_s$ as opposed to $f_t$ means that we can compute the value $\mathcal{V}_i:= \ell(f_s(x_i'),y_i)$ of an incoming data point. This value $\mathcal{V}_i$ acts as an estimate of the value of interest $v_i:=\ell(f_t(x_i'),y_i)$. The {\em stochastic $k$-secretary problem} is then to pick, under the noise model induced by using $f_s$, the optimal subset $S_{\mathcal{A}}$ of size $k$ from $\mathcal{D}$. 
Thus, with no further assumptions on $f_s$ it is unclear whether online algorithms, as defined in ~\S\ref{virtual_plus}, are still serviceable under uncertainty.

\xhdr{Sources of randomness}
Our method relies on the idea that we can use the surrogate model $f_s$ to estimate the value of some adversarial examples on the target model $f_t$. We justify here how partial knowledge on $f_t$ could provide us an estimate of $v_i$.
For example, we may know the general architecture and training procedure of $f_t$, but there will be inherent randomness in the optimization (e.g., due to initialization or data sampling), making it impossible to perfectly replicate $f_t$.

Moreover, it has been observed that, in practice, adversarial examples \emph{transfer} across models~\citep{papernot2016transferability, tramer2017space}.
In that context, it is reasonable to assume that the random variable $\mathcal{V}_i := \ell(f_s(x'_i),y_i)$ is likely to be close to $v_i := \ell(f_t(x'_i),y_i)$. 
We formalize this idea in Assumption~\ref{assump:feas}

\subsection{Stochastic Secretary Algorithms}
\vspace{-5pt}
\label{stochastic_secretary_algorithms_section}
In the stochastic $k$-secretary problem, we assume access to random variables $\mathcal{V}_i$ and that $v_i$ are fixed for $i = 1, \ldots,n$ and the goal is to maximize a notion of stochastic competitive ratio. This notion is similar to the standard competitive ratio defined in~\eqref{eq:comp_ration} with a minor difference that in the stochastic case, the algorithm does not have access to the values $v_i$ but to $\mathcal{V}_i$ that is an estimate of $v_i$. An algorithm is said to be $C_s$-competitive in the stochastic setting if asymptotically in $n$,
\begin{equation*}
    \mathbb{E}_{\pi \sim \mathcal{S}_n}[\setvaluemath(S_{\mathcal{A}})] \geq ( C_s + o(1)) \setvaluemath(S^*) \,.
    \label{eq:sto_comp_ration}
\end{equation*}
Here the expectation is taken over $\mathcal{S}_n$ (uniformly random permutations of the datastream $\mathcal{D}$ of size $n$) and over the randomness of $\mathcal{V}_i\,,\,i=1,\ldots,n$. $S_{\mathcal{A}}$ and $S^*$ are the set of items chosen by the stochastic online and offline algorithms respectively (note that while the online algorithm has access to $\mathcal{V}_i$, the offline algorithm picks the best $v_i$) and \setvalue\ is a set-value function as defined previously. 

\xhdr{Analysis of algorithms}
In the stochastic setting, all online algorithms observe $\mathcal{V}_i$ that is an estimate of the actual value $v_i$. Since the goal of the algorithm is to select the $k$-largest values by only observing random variables $(\mathcal{V}_i)$ it is requisite to make a feasibility assumption on the relationship between values $v_i$ and $\mathcal{V}_i$.\cut{ as otherwise the question of selecting top-$k$, from a theoretical perspective, is ill-posed.} Let us denote $\topk\{v_i\}$ as the set of top-$k$ values among $(v_i)$.
\begin{assumption}[Feasibility]\label{assump:feas}
 $\exists \gamma>0$ such that $\sP[\mathcal{V}_{i} \in \topk\{\mathcal{V}_i\}\,|\,v_i \in \topk\{v_i\}] \geq \gamma,\,\forall n\geq 0$.\footnote{Note that, for the sake of simplicity, the constant $\gamma$ is assumed to be independent of $n$ but a similar non-asymptotic analysis could be performed by considering a non-asymptotic definition of the competitive ratio.}
\end{assumption}
Assumption \ref{assump:feas} is a feasibility assumption as if the ordering of $(\mathcal{V}_{i})$ does not correspond at all with the ordering of $(v_i)$ then there is no hope any algorithm---online or an offline oracle---would perform better than random when picking $k$ largest $v_i$ by only observing $(\mathcal{V}_{i})$. In the context of adversarial attacks, such an assumption is quite reasonable as in practice there is strong empirical evidence between the transfer of adversarial examples between surrogate and target models (see~\S\ref{app:empirical_quantification_assumption_one}, for the empirical caliber of assumption \ref{assump:feas}). We can bound the competitive ratio in the stochastic setting.

\begin{theorem}\label{thm:stochastic_secretary} 
Let us assume that \algoname\ observes independent random variables $\mathcal{V}_{i}$ following Assumption~\ref{assump:feas}. Its stochastic competitive ratio $C_{s}$ can be bounded as follows,
\begin{equation}
  C \geq C_{s} \geq \gamma C 
\end{equation}
\end{theorem}
The proof of Thm.~\ref{thm:stochastic_secretary} can be found in \S\ref{appendix:proof_thm2}. Such a theoretical result is quite interesting as the stochastic setting initially appears significantly more challenging due to the non-zero probability that the observed ordering of historical values, $\mathcal{V}_i$, not being faithful to the true ranking based on $v_i$.

\subsection{Results on Synthetic Data}
\vspace{-5pt}
\looseness=-1
\label{results_on_synthetic_data}
We assess the performance of classical single threshold online algorithms and \algoname\ in solving the stochastic $k$-secretary problem on a synthetic dataset of size $n=100$ with $k \in [1,10]$. The value of a data point is its index in $\mathcal{D}$ prior to applying any permutation $\pi \sim \mathcal{S}_n$ plus noise $\mathcal{N}(0, \sigma^2)$. We compute and plot the competitive ratio over $10k$ unique permutations of each algorithm in Figure \ref{fig:synthetic_data}. 

\begin{minipage}[t]{.48\textwidth}
\begin{figure}[H]
     \vspace{-5pt}
    \includegraphics[width=1.01\linewidth]{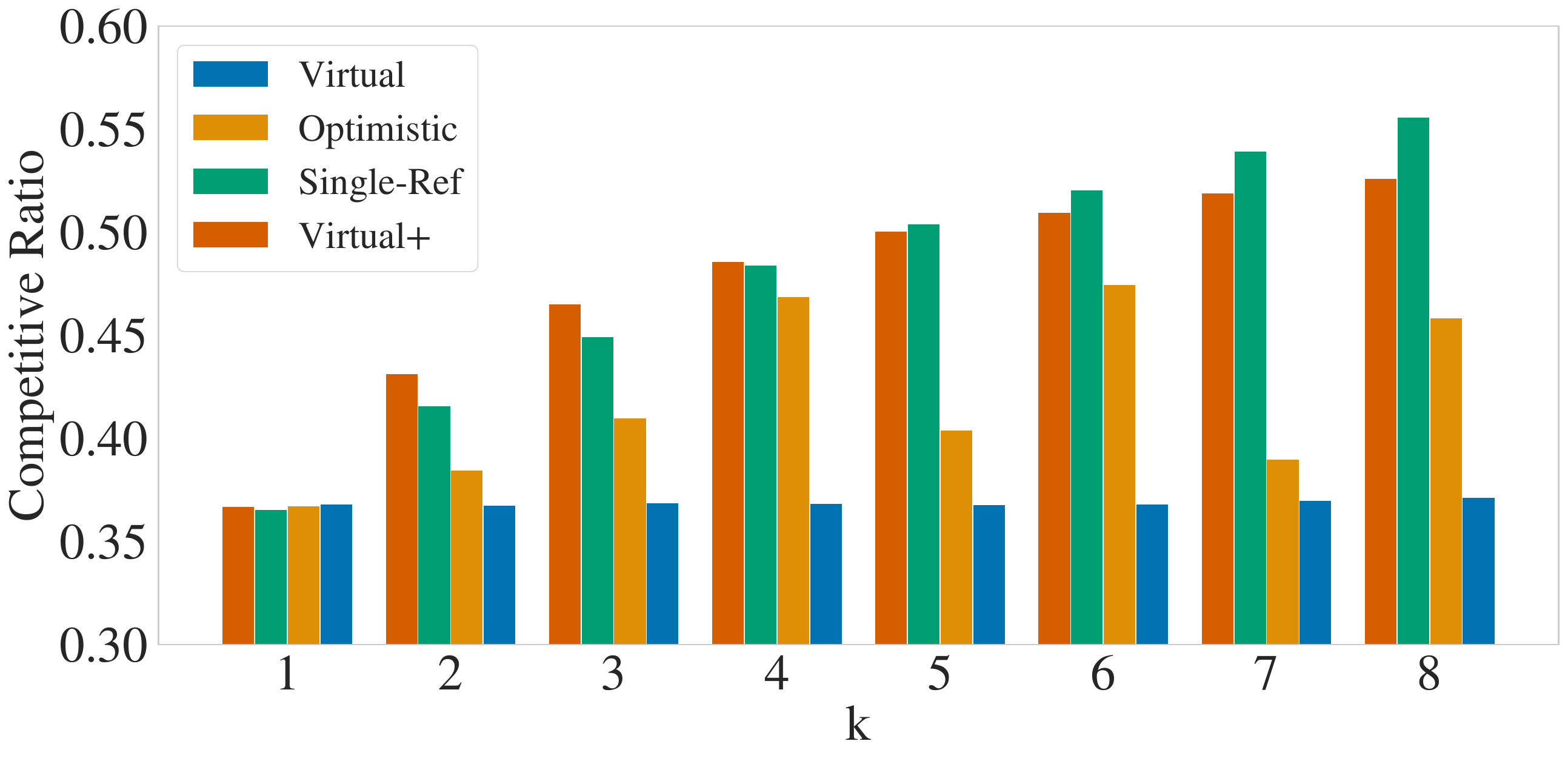}
    \vspace{-15pt}
    \caption{Estimation of the competitive ratio of online algorithms in the stochastic $k$-secretary problem with $\sigma^2=10$.}
    \label{fig:synthetic_data}
\end{figure}
\end{minipage}
\hfill
\begin{minipage}[t]{.48\textwidth}
\vspace{-5pt}  
    \begin{algorithm}[H]
    \small
    \textbf{Inputs:} Permuted Datastream: $\mathcal{D}_\pi$, Online Algorithm: $\mathcal{A}$,  Surrogate classifier: $f_s$, Target classifier: $f_t$,  Attack method: \textsc{Att}, Loss: $\ell$,  Budget: $k$,  
    Online Fool rate: $F^{\mathcal{A}}_\pi=0$.
    \begin{algorithmic}[1]
    \FOR{$(x_i,y_i)$ in $\mathcal{D}_\pi$}
    \STATE $x_i' \leftarrow  \textsc{Att}(x_i)$ \hfill \COMMENT{// Compute the attack}
    \STATE $\mathcal{V}_i \leftarrow \ell(f_s(x_i'),y_i)$  \hfill \COMMENT{// Estimate $v_i$}
    \IF{$\mathcal{A}(\mathcal{V}_1,\ldots, \mathcal{V}_i,k) == \textsc{True}$}
    \STATE  $F^{\mathcal{A}}_\pi \leftarrow F^{\mathcal{A}}_\pi+\tfrac{\mathbf{1}\{f_t(x_i')\neq y_i\}}{k}$  \hfill\COMMENT{// Submit  $x_i'$} 
    \ENDIF
    \ENDFOR
    \STATE \textbf{return:} $F^{\mathcal{A}}_\pi$ \hfill\COMMENT{//$\mathcal{A}$ always submits $k$ attacks} 
    \end{algorithmic}
     \caption{\small Online Adversarial Attack}
     \label{alg:online_adv_attack}
    \end{algorithm}
\end{minipage}

As illustrated for $k=1$ all algorithms achieve the optimal $(1/e)$-deterministic competitive ratio in the stochastic setting. Note that the noise level, $\sigma^2$, appears to have a small impact on the performance of the algorithms (\S\ref{appendix:synthetic_additional_results}). 
This substantiates our result in Thm.~\ref{thm:stochastic_secretary} indicating that $C_n$-competitive algorithms only degrade by a small factor in the stochastic setting. For $k<5$, \algoname\ achieves the best competitive ratio---empirically validating Thm \ref{thm:general_k_theorem1}---after which \textsc{Single-Ref} is superior.

\section{Experiments}
\looseness=-1
\vspace{-10pt}
\label{section:experiments}
We investigate the feasibility of online adversarial attacks by considering an online version of the challenging NoBox setting \citep{bose2020adversarial} in which the adversary must generate attacks without any access, including queries, to the target model $f_t$. Instead, the adversary only has access to a surrogate $f_s$ which is similar to $f_t$. In particular, we pick at random a $f_t$ and $f_s$ from an ensemble of pre-trained models from various canonical architectures. We perform experiments on the MNIST \cite{lecun-mnisthandwrittendigit-2010} and CIFAR-10 \cite{krizhevsky2009learning} datasets where we simulate a $\mathcal{D}$ by generating $1000$ permutations of the test set and feeding each instantiation to Alg.~\ref{alg:online_adv_attack}. In practice, online adversaries compute the value 
$\mathcal{V}_i=\ell(f_s(x'_i),y_i)$ 
of each data point in $\mathcal{D}$ by attacking $f_s$ using their fixed attack strategy
(where $\ell$ is the cross-entropy),
but the decision to submit the attack to $f_t$ is done using an online algorithm~$\mathcal{A}$ (see Alg.~\ref{alg:online_adv_attack}). 
As representative attack strategies, we use the well-known FGSM attack \citep{goodfellow2014explaining} and a universal whitebox attack in PGD \citep{madry2017towards}. 
We are most interested in evaluating the online fool rate, which is simply the ratio of successfully executed attacks against $f_t$ out of a possible of $k$ attacks selected by $\mathcal{A}$. \cut{Attacks are conducted with respect to the $\ell_\infty$- norm with $\gamma = 0.3$ for MNIST and $\gamma = 0.03125$ for CIFAR-10.} The architectures used for $f_s$, $f_t$, and additional metrics (e.g. competitive ratios) can be found in \S\ref{appendix:additional_results} \footnote{Code can be found at: \texttt{https://github.com/facebookresearch/OnlineAttacks}}.

\cut{
\begin{minipage}[t]{.5\linewidth}
\vspace{0pt}  
    \begin{algorithm}[H]
    \small
    \textbf{Inputs:} Permuted Datastream:$\mathcal{D}_\pi$, \, Online Algorithm:$\mathcal{A}$, \, Surrogate classifier:$f_s$,\, Target classifier:$f_t$,\, Attack method:$\textsc{Att}$,\, Loss:$\ell$, \, Budget:$k$,\, 
    Fool rate: $F^{\mathcal{A}}_\pi=0$.
    \begin{algorithmic}[1]
    \FOR{$(x_i,y_i)$ in $\mathcal{D}_\pi$}
    \STATE $x_i' \leftarrow  \textsc{Att}(x_i)$ \hfill \COMMENT{// Compute the attack}
    \STATE $\mathcal{V}_i \leftarrow \ell(f_s(x_i'),y_i)$  \hfill \COMMENT{// Compute the estimate of $v_i$}
    \IF{$\mathcal{A}(\mathcal{V}_1,\ldots, \mathcal{V}_i,k) == \textsc{True}$}
    \STATE  $F^{\mathcal{A}}_\pi \leftarrow F^{\mathcal{A}}_\pi+\tfrac{\mathbf{1}\{f_t(x_i')\neq y_i\}}{k}$  \hfill\COMMENT{// Submit  $x_i'$ on $f_t$} 
    \ENDIF
    \ENDFOR
    \STATE \textbf{return:} $F^{\mathcal{A}}_\pi$ \hfill\COMMENT{// Note that $\mathcal{A}$ always submits $k$ attacks} 
    \end{algorithmic}
     \caption{\small Online Adversarial Attack}
     \label{alg:online_adv_attack}
    \end{algorithm}
\end{minipage}
}

\xhdr{Baselines} We  rely  on  two  main baselines, first we use a \textsc{naive} baseline--a lower bound--where the data points are picked uniformly at random, and an upper bound with the \textsc{OPT} baseline where attacks, while crafted using $f_s$, are submitted by using the true value $v_i$ and thus utilizing $f_t$.%
\setlength{\textfloatsep}{5pt}
\begin{figure*}[ht!]
    \centering
    \vspace{-10pt}
    \includegraphics[width=.9\linewidth]{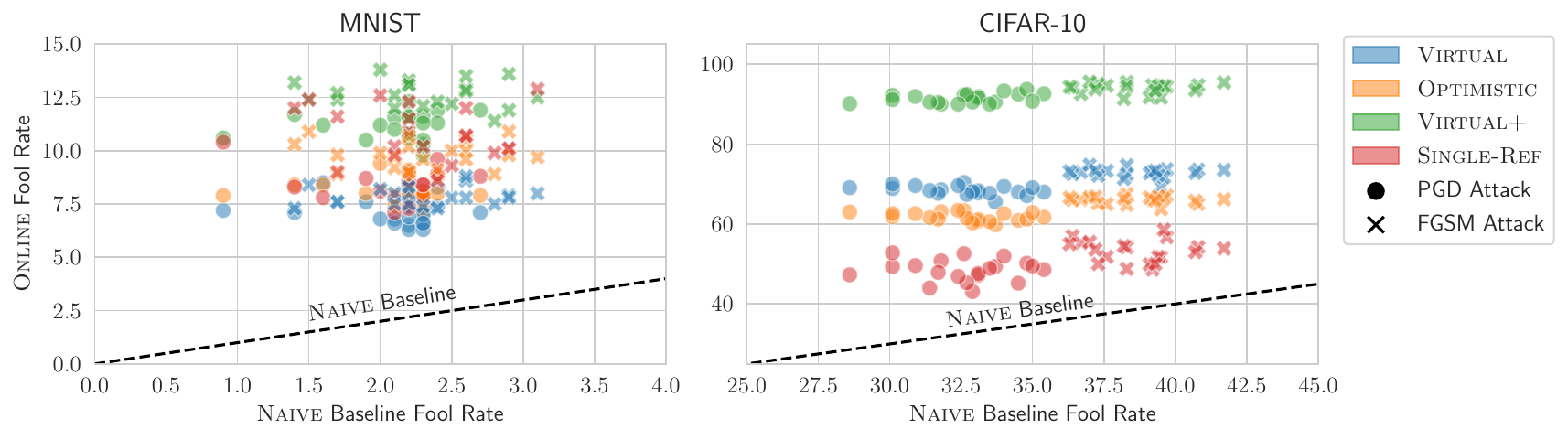}
    \vspace{-10pt}
    \caption{ \small
    Plot of online fool rates for $k=1000$ against PGD-robust models using different online algorithms $\mathcal{A}$, attacks, datasets, and $20$ different permutations. For a given $x$-coordinate, a higher $y$-coordinate is better.} 
    \label{fig:scatter}
\end{figure*}

\xhdr{Q1: Utility of using an online algorithm}
We first investigate the utility of using an online algorithm, $\mathcal{A}$, in selecting data points to attack in comparison to the \textsc{Naive} baseline. For a given permutation $\pi$ and an attack method (FGSM or PGD), we compute the online fool rate of the \textsc{Naive} baseline and an $\mathcal{A}$ as $F^{\textsc{Naive}}_\pi$, $F^{\mathcal{A}}_\pi$ respectively. In Fig.~\ref{fig:scatter}, we uniformly sample $20$ permutations $\pi_i \sim \mathcal{S}_n, \, i \in [n],$ of $\mathcal{D}$ and plot a scatter graph of points with coordinates $(F^{\textsc{Naive}}_{\pi_i}, F^{\mathcal{A}}_{\pi_i})$, for different $\mathcal{A}$'s, attacks with $k=1000$,\cut{\footnote{\small \url{https://github.com/MadryLab/[x]_challenge}, for $\texttt{[x]} \in \texttt{\{MNIST, CIFAR10 \}}$ .}} and datasets. The line $y=x$ corresponds to the \textsc{Naive} baseline performance ---i.e. coordinates $(F^{\textsc{Naive}}_\pi, F^{\textsc{Naive}}_\pi)$---and each point above that line corresponds to an $\mathcal{A}$ that outperforms the baseline on a given $\pi_i$. As observed, all $\mathcal{A}$'s significantly outperform the \textsc{Naive} baseline with an average aggregate improvement of 7.5\% and 34.1\% on MNIST and CIFAR-10.

\begin{table*}[ht]
\scriptsize
\label{table:non_robust_table1}
 \begin{center}\begin{tabular}{ c c c c c c c c c c c}
 \toprule
 & & \multicolumn{3}{c}{MNIST (Online fool rate in \%)} & \multicolumn{3}{c}{CIFAR-10  (Online fool rate in \%)} & \multicolumn{3}{c}{Imagenet  (Online fool rate in \%)}\\
 & Algorithm & \!$k=10$\! & \!$k=10^2$\! & \!$k=10^3$\! & \!$k=10$\! & \!$k=10^2$\! & \!$k=10^3$\! & \!$k=10$\! & \!$k=10^2$\! & \!$k=10^3$\!  \\
 \midrule
 \multirow{9}{*}{\rotatebox[origin=c]{90}{FGSM}}
 & \textsc{Naive}& 64.1 & 47.8 & 45.7  & 60.7 & 59.2 & 59.2 & 66.0 & 66.3 & 65.0\\
 & \textsc{Opt}  & \textbf{87.0} & \textbf{84.7 } &  \textbf{83.6 } & \textbf{86.6} & \textbf{87.3} &  \textbf{86.5} & \textbf{98.7} & \textbf{95.3} & \textbf{96.2} \\
 \cmidrule{2-11}
 & \textsc{Optimistic} & \cellcolor{g3}79.0 & \cellcolor{g3}77.6 &\cellcolor{g3} 75.3 &\cellcolor{g4} 75.3 & \cellcolor{g4} 72.8&\cellcolor{g4} 71.9 &  \cellcolor{g4} 86.0 &  \cellcolor{g4} 80.4 & \cellcolor{g4} 79.9 \\
 & \textsc{Virtual} & \cellcolor{g3}78.6 &\cellcolor{g3} 79.1  &\cellcolor{g3} 77.4 & \cellcolor{g4} 76.1 &\cellcolor{g4} 77.1  & \cellcolor{g4}75.4 &  \cellcolor{g4} 85.3 &  \cellcolor{g4} 84.9 & \cellcolor{g4} 84.3\\
 & \textsc{Single-Ref} &\cellcolor{g2}85.1 & \cellcolor{g1}83.0$^*$ &\cellcolor{g4} 72.3 &\cellcolor{g3} 80.4 &\cellcolor{g2} 84.0 & \cellcolor{g4}66.0 &  \cellcolor{g2}  94.0$^*$ &  \cellcolor{g2} 92.4$^*$ & \cellcolor{g4} 72.5 \\
 & \algoname & \cellcolor{g3}80.4 &\cellcolor{g1} 82.5$^*$ & \cellcolor{g1} 82.9 &\cellcolor{g2} 82.9  & \cellcolor{g1}86.3 &  \cellcolor{g1}85.2 &  \cellcolor{g1} 96.0$^*$ &  \cellcolor{g1} 95.0$^*$ & \cellcolor{g1} 95.8\\
 \midrule
  \multirow{9}{*}{\rotatebox[origin=c]{90}{PGD}}
 & \textsc{Naive} & 69.7  & 67.2 & 67.9 & 72.5 & 70.4 & 68.6 & 72.5 & 72.5 & 73.8 \\
 & \textsc{Opt} & \textbf{73.6 } & \textbf{49.8 } & \textbf{49.6} & \textbf{83.7 } & \textbf{80.6 } & \textbf{79.9} & \textbf{82.5} & \textbf{80.2} & \textbf{76.8}\\
 \cmidrule{2-11}
 & \textsc{Optimistic} & \cellcolor{g3} 66.2 & \cellcolor{g2} 48.2  &\cellcolor{g3} 45.1 &\cellcolor{g3} 79.1 & \cellcolor{g3}76.6 & \cellcolor{g3}76.0 & \cellcolor{g1} 87.5$^*$ & \cellcolor{g1} 78.0$^*$ &\cellcolor{g1} 74.5$^*$\\
 & \textsc{Virtual} &\cellcolor{g4} 63.4 & \cellcolor{g4}46.2 &\cellcolor{g2} 46.8 & \cellcolor{g3}78.3 &\cellcolor{g3} 77.5 &\cellcolor{g3} 76.9 & \cellcolor{g2} 80.0$^*$ & \cellcolor{g3} 74.0$^*$ &\cellcolor{g1} 75.6$^*$ \\
 &\textsc{Single-Ref} & \cellcolor{g1} 71.5 & \cellcolor{g1} 49.7$^*$ & \cellcolor{g4}42.9 & \cellcolor{g2}80.2$^*$ &\cellcolor{g1} 79.6$^*$ & \cellcolor{g4} 74.5 & \cellcolor{g3} 77.5$^*$ & \cellcolor{g1} 79.5$^*$ & \cellcolor{g1}75.2$^*$ \\
 & \algoname         &\cellcolor{g2} 68.2 & \cellcolor{g1}49.3$^*$ & \cellcolor{g1} 49.7 & \cellcolor{g1} 81.2$^*$ & \cellcolor{g1} 80.1$^*$ &\cellcolor{g1}79.5 & \cellcolor{g3} 77.5$^*$ & \cellcolor{g1} 79.0$^*$ & \cellcolor{g1}76.4$^*$ \\
 \bottomrule
\end{tabular}\end{center}
\vspace{-10pt}
\caption{\small
Online fool rate of various online algorithms on non-robust models. For a given attack and value of $k$: {\color{g1} $\mathbf{\bullet}$ } at least 97\%,
\textbf{\color{g2} $\mathbf{\bullet}$} at least 95\%, \textbf{\color{g3}$\mathbf{\bullet}$} at least 90\%, \textbf{\color{g4} $\mathbf{\bullet}$} less than 90\% of the optimal performance. $^*$ indicates when there is several best methods with overlapping error bars. Detailed results with error bars can be found in~\S\ref{appendix:detailed_non_robust_results}.}
\end{table*}

\xhdr{Q2: Online Attacks on Non-Robust Classifiers}
\cut{
While previous work on online algorithms focused largely on the theoretical contribution, we are equally interested in their performance in practical applications, where the assumptions of the theory may not necessarily hold anymore.}We now conduct experiments on non-robust MNIST, CIFAR-10, and Imagenet classifiers. We report the average performance of all online algorithms, and the optimal offline algorithm \textsc{Opt} in Tab.~\ref{table:non_robust_table1}. For MNIST, we find that the two best online algorithms are \textsc{Single-Ref} and our proposed \algoname\ which approach the upper bound provided by \textsc{Opt}. For experiments with $k<5$ please see~\S\ref{app:additional_results_small_k}. For $k=10$ and $k=100$, \textsc{Single-Ref} is slightly superior while for $k=1000$ \algoname\ is the best method with an average relative improvement of $15.3\%$. This is unsurprising as \algoname\ does not have any additional hyperparameters unlike \textsc{Single-Ref} which appears more sensitive to the choice of optimal thresholds and reference ranks, both of which are unknown beyond $k=100$ and non-trivial to find in closed form (see \S\ref{appendix:additional_results_larger_datasets} for details). On CIFAR-10, we observe that \algoname\ is the best approach regardless of attack strategy and the online attack budget $k$. Finally, for ImageNet we find that all online algorithms improve over the \textsc{Naive} baseline and approach saturation to the optimal offline algorithm, and as a result, all algorithms are equally performant---i.e. within error bars (see~\S\ref{appendix:detailed_non_robust_results} for more details). A notable observation is that even conventional whitebox adversaries like FGSM and PGD become strong blackbox transfer attack strategies when using an appropriate $\mathcal{A}$.

\begin{table*}[ht]
\footnotesize
\label{table:madry_challenge}
 \begin{center}\begin{tabular}{ c c c c c c c c}
 \toprule
 & & \multicolumn{3}{c}{MNIST (Online fool rate in \%)} & \multicolumn{3}{c}{CIFAR-10 (Online fool rate in \%)}\\
 & Algorithm & $k=10$ & $k=100$ & $k=1000$ & $k=10$ & $k=100$ & $k=1000$ \\
 \midrule 
 \multirow{6}{*}{\rotatebox[origin=c]{90}{FGSM}}
 & \textsc{Naive} & $2.1 \pm 4.5$ & $2.1 \pm 1.4$ & $2.1 \pm 0.4$ & 31.9 $\pm$ 14.2 & 32.6 $\pm$ 4.7 & 32.5 $\pm$ 1.5\\
 & \textsc{Opt} & \textbf{80.0 $\pm$ 0.0} & \textbf{55.0 $\pm$ 0.0} & \textbf{18.9 $\pm$ 0.0} & \textbf{100.0 $\pm$ 0.0} & \textbf{100.0 $\pm$ 0.0} & \textbf{97.2 $\pm$ 0.0}\\
 \cmidrule{2-8}
 & \textsc{Optimistic} & \cellcolor{g3}$49.7 \pm 0.6$ & \cellcolor{g4}$25.7$ $\pm$ $0.1$ &\cellcolor{g3} $9.7$ $\pm$ $0.0$ & \cellcolor{g2}72.4 $\pm$ 0.5 & \cellcolor{g3}64.6 $\pm$ 0.1 & \cellcolor{g3}61.9 $\pm$ 0.0\\
 & \textsc{Virtual} & \cellcolor{g3}49.8 $\pm$ 0.5 & \cellcolor{g3}27.8 $\pm$ 0.1 &\cellcolor{g4} 8.1 $\pm$ 0.0 & \cellcolor{g2}75.1 $\pm$ 0.5 & \cellcolor{g3}74.3 $\pm$ 0.1 & \cellcolor{g2}68.9 $\pm$ 0.0\\
 & \textsc{Single-Ref} & \cellcolor{g2}62.0 $\pm$ 0.7 & \cellcolor{g2}45.2 $\pm$ 0.2 & \cellcolor{g3}10.2 $\pm$ 0.0 & \cellcolor{g2}84.3 $\pm$ 0.6 & \cellcolor{g1}90.9 $\pm$ 0.3 &\cellcolor{g4} 48.6 $\pm$ 0.1\\
 & \algoname & \cellcolor{g2}68.2 $\pm$ 0.5 & \cellcolor{g2}42.2 $\pm$ 0.1 & \cellcolor{g3}{12.7 $\pm$ 0.0} & \cellcolor{g1}91.5 $\pm$ 0.4 & \cellcolor{g1}96.5 $\pm$ 0.1 & \cellcolor{g1}{91.7 $\pm$ 0.0}\\
 \midrule
 \multirow{6}{*}{\rotatebox[origin=c]{90}{PGD}}
 & \textsc{Naive} & 1.8 $\pm$ 4.1 & 1.9 $\pm$ 1.4 & 1.9 $\pm$ 0.4 & 39.1 $\pm$ 14.2 & 38.9 $\pm$ 4.4 & 38.7 $\pm$ 1.5\\
 & \textsc{Opt} & \textbf{58.9 $\pm$ 0.4} & \textbf{39.9 $\pm$ 0.1} & \textbf{16.1 $\pm$ 0.0 }& \textbf{100.0 $\pm$ 0.0} & \textbf{100.0 $\pm$ 0.0} & \textbf{98.0 $\pm$ 0.0}\\
 \cmidrule{2-8}
 & \textsc{Optimistic} & \cellcolor{g2}34.9 $\pm$ 0.5 & \cellcolor{g4}19.2 $\pm$ 0.1 &\cellcolor{g3} 8.2 $\pm$ 0.0 & \cellcolor{g2}75.4 $\pm$ 1.9 & \cellcolor{g3}68.5 $\pm$ 0.4 & \cellcolor{g3}66.0 $\pm$ 0.1\\
 & \textsc{Virtual} & \cellcolor{g2}35.4 $\pm$ 0.5 & \cellcolor{g3}21.8 $\pm$ 0.1 &\cellcolor{g4} 7.2 $\pm$ 0.0 & \cellcolor{g2}78.1 $\pm$ 1.7 &\cellcolor{g2} 77.3 $\pm$ 0.5 & \cellcolor{g2}72.8 $\pm$ 0.1\\
 & \textsc{Single-Ref} &\cellcolor{g1} 44.1 $\pm$ 0.6 & \cellcolor{g2}33.9 $\pm$ 0.2 & \cellcolor{g3}8.3 $\pm$ 0.0 & \cellcolor{g2}86.2 $\pm$ 2.2 &\cellcolor{g1} 91.9 $\pm$ 0.9 & \cellcolor{g3}53.2 $\pm$ 0.3\\
 & \algoname & \cellcolor{g1}48.3 $\pm$ 0.5 & \cellcolor{g2}32.8 $\pm$ 0.1 & \cellcolor{g3}11.1 $\pm$ 0.0 & \cellcolor{g1}92.2 $\pm$ 1.3 & \cellcolor{g1}97.1 $\pm$ 0.4 &\cellcolor{g1} 94.2 $\pm$ 0.1\\
 \bottomrule
\end{tabular}\end{center} 
\vspace{-10pt}
\caption{\small
Online fool rate of various online algorithms on robust models. For a given attack and value of $k$: {\color{g1} $\mathbf{\bullet}$ } at least 90\%,
\textbf{\color{g2} $\mathbf{\bullet}$} at least 80\%, \textbf{\color{g3}$\mathbf{\bullet}$} at least 70\%, \textbf{\color{g4} $\mathbf{\bullet}$} less than 70\% of the optimal performance.}
\end{table*}

\xhdr{Q3: Online Attacks on Robust Classifiers}
We now test the feasibility of online attacks against classifiers robustified using adversarial training by adapting the public Madry Challenge \citep{madry2017towards} to the online setting. We report the average performance of each $\mathcal{A}$ in Table ~\ref{table:madry_challenge}. We observe that \algoname\ is the best online algorithms, outperforming \textsc{Virtual} and \textsc{Optimistic}, in all settings except for $k=10$ on MNIST where \textsc{Single-Ref} is slightly better.

\xhdr{Q4: Differences between the online and offline setting} The online threat model presents several interesting phenomena that we now highlight. First, we observe that a stronger attack (e.g. PGD)---in comparison to FGSM---in the offline setting doesn't necessarily translate to an equivalently stronger attack in the online setting. Such an observation was first made in the conventional offline transfer setting by \citet{madry2017towards}, but we argue the online setting further exacerbates this phenomenon. We explain this phenomenon in Fig.~\ref{fig:fgsm_vs_pgd_loss} \&~\ref{fig:fgsm_vs_pgd_ratio} by plotting the ratio of unsuccessful attacks to total attacks as a function of loss values for PGD and FGSM.  We see that for the PGD attack numerous unsuccessful attacks can be found even for high surrogate loss values and as a result, can lead $\mathcal{A}$ further astray by picking unsuccessful data points---which may be top-$k$ in surrogate loss values---to conduct a transfer attack. 
\cut{
We observe that the PGD curve is above the FGSM curve indicating that while PGD is a stronger attacker than FGSM, it is harder for $\mathcal{A}$ to pick successful attacks using PGD.} 
A similar counter-intuitive observation can be made when comparing the online fool rate on robust and non-robust classifiers. While it is natural to expect the online fool rate to be lower on robust models we empirically observe the opposite in Tab.~\ref{table:non_robust_table1} and~\ref{table:madry_challenge}. To understand this phenomenon we plot the ratio of unsuccessful attacks to total attacks as a function $f_s$'s loss in Fig.~\ref{fig:histogram_Vi} and observe non-robust models provide a non-vanishing ratio of unsuccessful attacks for large values of $\mathcal V_i$ making it harder for $\mathcal{A}$ to pick successful attacks purely based on loss (see also~\S\ref{appendix:visualization_of_values_observed}). 
\begin{figure}[h]
\vspace{-10pt}
    \centering
    \captionsetup[subfigure]{justification=centering}
    \begin{subfigure}[b]{0.32\columnwidth}
    \centering
    \includegraphics[width=\textwidth]{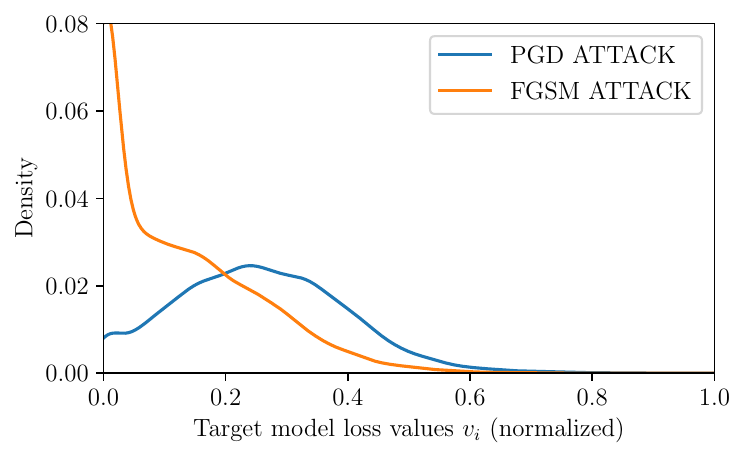}
    \caption{\small Distribution of $f_t$'s loss values.}
    \label{fig:fgsm_vs_pgd_loss}
    \end{subfigure}
    \hfill
    \begin{subfigure}[b]{0.32\columnwidth}
    \centering
    \includegraphics[width=\textwidth]{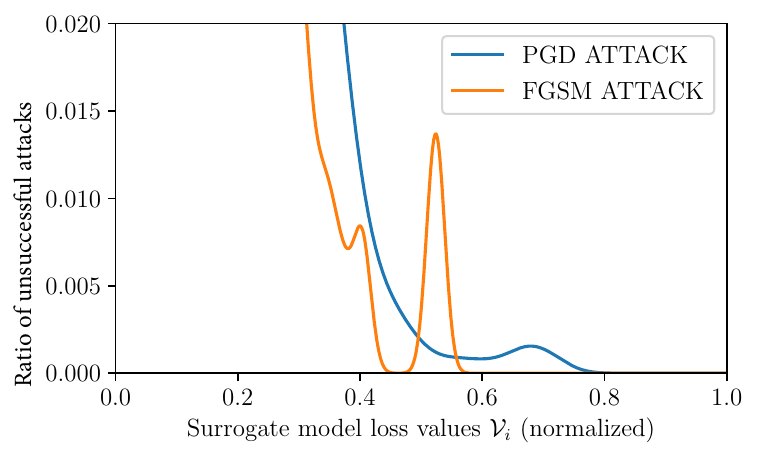}
    \caption{\small Ratio of unsuccessful attacks.}
    \label{fig:fgsm_vs_pgd_ratio}
    \end{subfigure}
    \hfill
    \begin{subfigure}[b]{0.32\columnwidth}
    \centering
    \includegraphics[width=\textwidth]{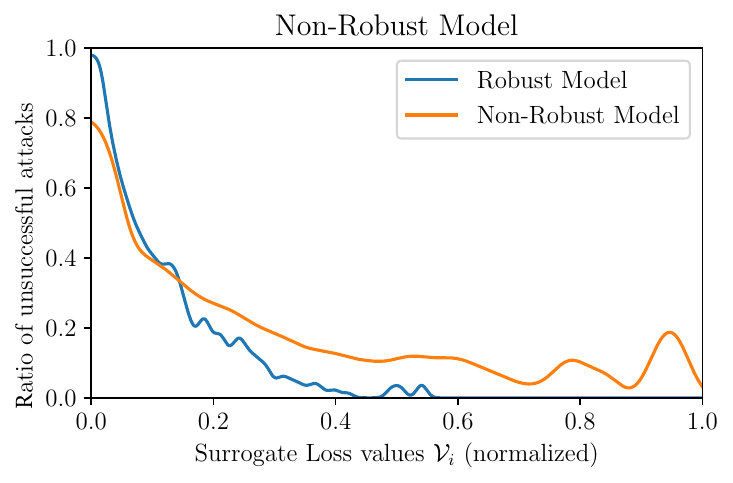}
    \caption{\small Ratio of unsuccessful attacks.}
    \label{fig:histogram_Vi}
    \end{subfigure}
    \caption{\small For every example in MNIST we compute an attack using $f_s$ and submit it to $f_t$. \textbf{Left}: The distribution of the normalized loss values of $f_t$ for all attacks where a higher loss is a stronger attack. \textbf{Middle}: The percentage of unsuccessful attacks as a function of $f_s$ normalized loss values. \textbf{Right}: smoothed ratio of unsuccessful attacks to total attacks as a function of the $f_s$ normalized loss values.}
    \label{fig:fgsm_vs_pgd}
    \vspace{-5pt}
\end{figure}
\vspace{-10pt}
\section{Conclusion}
\vspace{-10pt}
In this paper, we formulate the online adversarial attack problem, a novel threat model to study adversarial attacks on streaming data. We propose \algoname, a simple yet practical online algorithm that enables attackers to select easy to fool data points while being theoretically the best single threshold algorithm for $k<5$. We further introduce the stochastic $k$-secretary problem and prove fundamental results on the competitive ratio of any online algorithm operating in this new setting. Our work sheds light on the tight coupling between optimally selecting data points using an online algorithm and the final attack success rate, enabling weak adversaries to perform on par with stronger ones at no additional cost. Investigating, the optimal threshold values for larger values of $k$ along with competitive analysis for the general setting is a natural direction for future work.  

\section*{Ethics Statement}
\vspace{-10pt}
\label{broader_impact}
We introduce the online threat model which aims to capture a new domain for adversarial attack research against streaming data. Such a threat model exposes several new security and privacy risks. For example, using online algorithms, adversaries may now tailor their attack strategy to attacking a small subset of streamed data but still cause significant damage to downstream models e.g. the control system of an autonomous car. On the other hand our research also highlights the need and importance of stateful defence strategies that are capable of mitigating such online attacks. On the theoretical side the development and analysis of \algoname \ has many potential applications outside of adversarial attacks broadly categorized as resource allocation problems. As a concrete example one can consider advertising auctions which provide the main source of monetization for a variety of internet services including search engines, blogs, and social networking sites. Such a scenario is amenable to being modelled as a secretary problem as an advertiser may be able to estimate accurately the bid required to win a particular auction, but may not be privy to the trade off for future auctions.

\section*{Reproducibility statement}
\vspace{-10pt}
\label{broader_impact}
Throughout the paper we tried to provide as many details as possible in order for the results of the paper to be reproducible. In particular, we provide a detailed description of \algoname in Alg.~\ref{alg:virtual_plus} and we explain how to combine any attacker (e.g. PGD) with an online algorithm to form an online adversarial attack in Alg.~\ref{alg:online_adv_attack}. We provide a general description of the experimental setup in~\S\ref{section:experiments}, further details with the specific architecture of the models and hyper-parameters used are provided in~\S\ref{appendix:additional_results_larger_datasets}. We also provided confidence intervals with our experiments every time it was possible to do so. Finally the code used to produce the experimental results is provided with the supplementary materials and will be made public after the review process.

\section*{Acknowledgements}

The authors would like to acknowledge Manuella Girotti, Pouya Bashivan, Reyhane Askari Hemmat, Tiago Salvador and Noah Marshall for reviewing early drafts of this work.

\textbf{Funding.} This work is partially supported by the Canada CIFAR AI Chair Program (held at Mila).
Joey Bose was also supported by an IVADO PhD fellowship.
Simon Lacoste-Julien and Pascal Vincent are CIFAR Associate Fellows in the Learning in Machines \& Brains program. Finally, we thank Facebook for access to computational resources.
\section*{Contributions}


\emph{Andjela Mladenovic} and \emph{Gauthier Gidel} formulated the online adversarial attacks setting by drawing parallels to the $k$-secretary problem, with \emph{Andjela Mladenovic} leading the theoretical investigation and theoretical results including the competitive analysis for \algoname\ for the general-$k$ setting. \emph{Avishek Joey Bose} conceived the idea of online attacks, drove the writing of the paper and helped \emph{Andjela Mladenovic} with experimental results on synthetic data. \emph{Hugo Berard} was the chief architect behind all experimental results on MNIST and CIFAR-10.  \emph{William L. Hamilton}, \emph{Simon Lacoste-Julien} and \emph{Pascal Vincent} provided feedback and guidance over this research while \emph{Gauthier Gidel} supervised the core technical execution of the theory.

\clearpage

\bibliography{iclr2022_conference}
\bibliographystyle{iclr2022_conference}

\appendix
\onecolumn
\section{Proof of Competitive Ratio for \algoname Algorithm}
\label{appendix:virtual_plus_proof_k_2}

As an illustrative example that aids in understanding the full general-$k$ proof for for the competitive ratio \algoname\ we now prove Theorem \ref{thm:general_k_theorem1} for $k=2$ from the main paper. 

\begin{theorem}  For $k=2$, the competitive ratio achieved by \algoname\ algorithm is equal to, 
\begin{equation}
    C_n = \frac{t(t-1)}{n} \sum_{j=t}^{n-1}\frac{1}{j(j-1)} \left(1 + 2 \sum_{ p=t+1}^j \frac{1}{p-1}\right)
\end{equation}
Particularly for $t= \alpha \cdot n\,,\, \alpha \in (0,1)$ we get 
\begin{equation} \label{eq:lower_bound_alpha1}
    C_n > \alpha ( 3(1-\alpha) + 2 \alpha \ln(\alpha)) + \mathcal{O}(1/n)
\end{equation}
Thus, asymptotically  we have
\begin{equation}
    C >  \max_{\alpha \in [0,1]} \alpha ( 3(1-\alpha) + 2 \alpha \ln(\alpha)) > .4273 > 1/e \, .
\end{equation}
\end{theorem}

\begin{figure}[ht]
    \centering
    \includegraphics[width=1.0\linewidth]{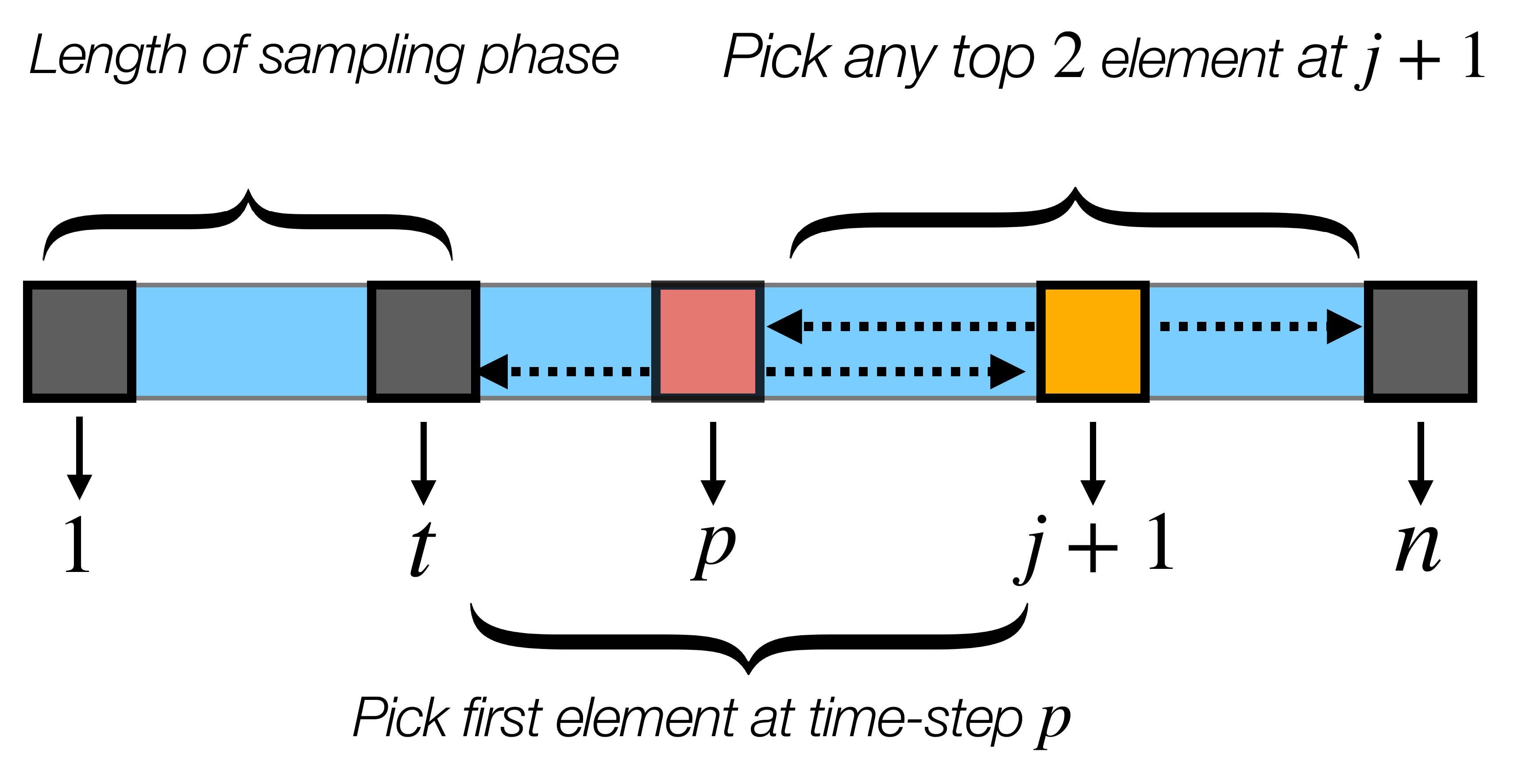}
    \caption{Probability of having only one element in $S_{\mathcal{A}}$ after $j$ time-steps with the \algoname\ algorithm.}
    \label{fig:k_2}
    \vspace{-15pt}
\end{figure}

\begin{proof}
First note that by \citet[Lemma 3.3]{albers2020new} we can show that the competitive ratio for the $k$-secretary problem for a monotone algorithm is equal to 
\begin{equation}
\label{eq:C_as_sum_k_two}
    C = \frac{1}{k}\sum_{a=1}^k \mathbb{P}(i_a \in S_\mathcal{A}), 
\end{equation}
where $i_a$ is the index of the $a^{th}$ secretary picked by the offline solution ---i.e. $i_a$ is a top-$k$ secretary of $\mathcal{D}$. By Lemma \ref{lemma_monotone} \algoname \ is a monotone algorithm and we may use ~\eqref{eq:C_as_sum_k_two}.
Now, let us focus on the case $k=2$.
When calculating the probability of either of the top two items in $\mathcal{D}$ being picked by the \algoname\ we must first compute the probability of one of the top-2 items being picked during the selection phase (time step $t+1 \dots n$). Now notice that \algoname\ picks an item at time step $j+1$ if and only if this is a top-2 item with respect to all of $\mathcal{D}$ and $|S_{\mathcal{A}}| \leq 2$ at time-step $j+1$. Let top-$2_j$ denote the two largest elements observed by $\mathcal{A}$ up to and inclusive of time step $j$. Thus, for $a \in \{1,2\}$, we have
\begin{align}
    \mathbb{P}(i_a \in S_\mathcal{A}) 
    &= \sum_{j=t}^{n-1} \mathbb{P}(i_a \in S_\mathcal{A} \text{ at time-step }j+1)\\ 
    &= \frac{1}{n}\sum_{j=t}^{n-1} \mathbb{P}(|S_{\mathcal{A}}| < 2 \text{ at time-step } j + 1) \notag
\end{align}

Now, we compute $\mathbb{P}(|S_{\mathcal{A}}| \leq 2 \text{ at time-step } j+1)$ by decomposing this probability into the following two events: A.) $|S_{\mathcal{A}}| = 0$ where the selection set is empty and B.) the event $|S_{\mathcal{A}}| = 1$ where exactly one item has been picked. We now analyze each event in turn.

\xhdr{Event A} In order for the event $|S_{\mathcal{A}}|=0$ to occur it implies that the algorithm does not select any items in the first $j$ rounds. This means both two top-$2_j$ elements must have appeared in the sampling phase. Thus the probability for this event is exactly $\frac{t(t - 1)}{j (j - 1)}$. 

\xhdr{Event B}
The second event is when $|S_{\mathcal{A}}|=1$ ---i.e. the algorithm picks exactly one element in the first $j$ rounds. The computation of this event is illustrated in Figure~\ref{fig:k_2}. Let's say that an element is picked at time step $p$. Now to compute the probability of Event B occurring we first make the following two observations:

\begin{enumerate}[label={\bf Observation \arabic*:}, topsep=0pt, parsep=0pt, leftmargin=69pt, itemsep=2pt]
    \item  In order for exactly one element to be picked at the time step $p \leq j$, this element must be one of the top-$2_j$ elements. Furthermore, this implies the other of the top-$2_j$ element ---i.e. the one not picked at $p$ must have appeared in the sampling phase. Note that if both top-$2_j$ elements appear after the sampling phase, the condition would be satisfied twice and two elements would be selected instead of exactly one, and if they both appeared during the sampling phase we return to Event A. As a result, the probability for this condition is given by $\frac{t}{j (j - 1)}$.  
    \item By observation 1. we know that the online algorithm $\mathcal{A}$ picks one of the top-$2_j$ at time step $p$ and the fact that the event under consideration is $|S_{\mathcal{A}}|=1$ the reference list $R$ from time step $p$ to $j+1$ must contain both top-$2_j$ elements. However, for $\mathcal{A}$ to pick \emph{only} at $p$ we also need to ensure that no elements are picked prior t to $p$. Therefore, before time step $p$ the reference list must contain  top-$2_{p}$ . Again by observation 1, we know that $R$ already contains one of the top-$2_j$ elements therefore we know it contains one of the  top-$2_p$ elements. Thus the probability of ensuring that the second  top-$2_p$ elements is also within $R$ by time step $p$ is $\frac{(t - 1)}{(p - 2)}$. Finally, since there are two top elements and they may appear in any order we must count the probability of Event B occurring twice.

\end{enumerate}

Overall we get:
\begin{equation}
    \frac{t(t - 1)}{j (j - 1)} + 2 \sum_{p = t + 1}^{j}\frac{1}{j}\frac{t}{j - 1}\frac{t-1}{p-2}
\end{equation}
Total probability:
\begin{align*}
    \frac{1}{n} \sum_{j = t}^{n-1}\left(\frac{t(t - 1)}{j (j - 1)} + 2 \sum_{p = t + 1}^{j}\frac{1}{j}\frac{t}{j - 1}\frac{t-1}{p-2}\right) 
    &= 
   \frac{1}{n}\sum_{j = t}^{n-1}\left(1 + 2\sum_{p=t}^{j - 1} \frac{1}{p - 1} \right) \\
   &= 
   \frac{t(t - 1)}{n}\sum_{j=t}^{n-1}\left(\frac{1}{j(j-1)} + \frac{2}{j(j-1)}\sum_{p = t}^{j-1}\frac{1}{p-1}\right)\\
   & >
   \frac{t(t-1)}{n}\sum_{j = t}^{n - 1}\left(\frac{1}{j^2} + 2\frac{1}{j^2}\sum_{p = t + 1}^{j}\frac{1}{p-1}\right) \\
   & >
   \frac{t(t-1)}{n}\sum_{j = t}^{n - 1}\left(\frac{1}{j^2} + \frac{2}{j^2}\int_{p=t+1}^{j+1} \frac{1}{p-1}\, dp \right)\\
   & >
   \frac{t(t-1)}{n}\sum_{j = t}^{n-1}\left(\frac{1}{j^2} + \frac{2}{j^2}\ln\left(\frac{j}{t}\right)\right)
\end{align*}
Now we will use the following lemma
\begin{lemma}
For any differentiable function $f$ and any $a<b$, we have,
\begin{equation}
    \sum_{j=a}^{b} f(j) \geq \int_a^{b+1} f(t) dt - |b+1-a|\sup_{t \in [a,b+1]} |f'(t)|
\end{equation}
\end{lemma}
\begin{proof}
\begin{equation}
     | f(n) - \int_n^{n+1} f(t) dt | \leq \int_n^{n+1} |f(n)-f(t)| dt \leq \sup_{t \in [n,n+1]} |f'(t)| 
\end{equation}
Thus, 
\begin{equation}
    f(n) \geq  \int_n^{n+1} f(t) dt - \sup_{t \in [n,n+1]} |f'(t)| 
\end{equation}
and by summing for $n = a \ldots b$ we get the desired lemma. 
\end{proof}
Applying this lemma to $f(x) = \frac{1+2\ln(x/t)}{x^2}\,, a = t$ and $b=n-1$, we get 
\begin{align}
    \frac{1}{n} \sum_{j = t}^{n-1}\frac{t(t - 1)}{j (j - 1)} + 2 \sum_{p = t + 1}^{j}\frac{1}{j}\frac{t}{j - 1}\frac{t-1}{p-2} 
    & > \frac{t(t-1)}{n}\sum_{j = t}^{n-1}\left(\frac{1}{j^2} + \frac{2}{j^2}\ln\left(\frac{j}{t}\right)\right) \\
    & \geq \frac{t(t-1)}{n}\left(\int_{t}^{n}  \frac{1+2\ln(x/t)}{x^2} dx - 2(n-t) \sup_{x \in [t,n]} \left| 
    \frac{4 \ln(x/t)}{x ^3}\right| \right) \notag \\
    & \geq  \frac{t(t-1)}{n}\left(\int_{t}^{n}  \frac{1+2\ln(x/t)}{x^2} dx - 2(n-t) \left| \frac{16}{3 t^3 e^4} \right| \right) \notag \\
    & = 
    \frac{t(t -1)}{n} \left(\frac{3}{t}-\frac{2\ln(n/t) + 3}{n} - 2(n-t) \left| \frac{16}{3 t^3 e^4} \right|\right) 
\end{align}
Now for $t = \alpha n$ where $\alpha \in (0, 1)$ and as $n \xrightarrow[]{} \infty$, that lower-bound becomes 
\begin{equation}
C \geq \alpha (3 - \alpha(3 - 2\ln (\alpha))) + \mathcal{O}(1/n) \,,\quad \forall \alpha \in (0,1)
\end{equation}
The constant term of the RHS is a concave function of $\alpha$ that is maximized for $\alpha^* \approx 0.38240$. Thus, our algorithms achieves competitive ratio larger than $0.42737$.
\end{proof}

\clearpage

\section{Competitive Ratio General $k$}
\label{app:general_k_proof}
We now prove our main result for the competitive ratio of \algoname \ for $k \geq 2$. The theorem statement is reproduced here for convenience.  
\begin{reptheorem}{thm:general_k_theorem1}
The competitive ratio of \algoname\ for $k \geq 2$ with threshold $t_k = \alpha n$ can asymptotically be lower bounded by the following concave optimization problem,
\begin{equation}
    C_k >  \max_{\alpha \in [0,1]} f(\alpha) := {\alpha}^k \left(\sum_{m = 0}^{k - 1} a_m \ln^m (\alpha)\right) - \alpha a_0
    \hspace{0.15cm }where  \hspace{0.15cm }
    a_m = \left(\frac{\frac{k^k}{(k-1)^{k-m}} - k^m}{m!}\right)(-1)^{m+1}
    \, . \notag
\end{equation}
Particularly, we get $C_2\geq0.427, C_3\geq .457, C_4\geq.4769$ outperforming~\citet{albers2020new}.
\end{reptheorem}
\begin{figure}[ht]
    \centering
    \includegraphics[width=1.0\linewidth]{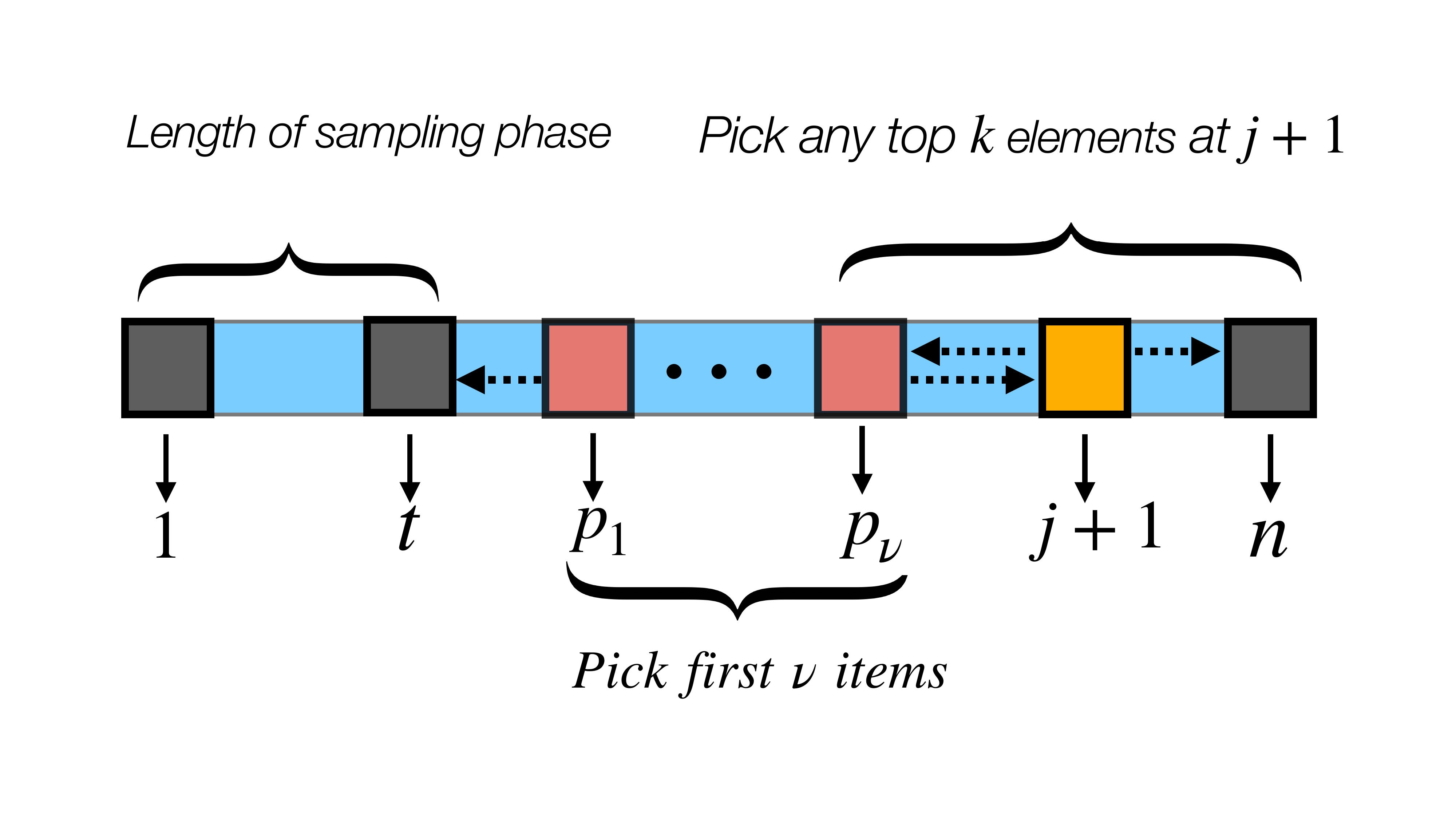}
    \caption{Virtual+ $k \geq 2$ proof.}
    \label{fig:general_k}
    \vspace{-15pt}
\end{figure}
\begin{proof}
First note that by \citet[Lemma 3.3]{albers2020new} we can show that the competitive ratio for the $k$-secretary problem for a monotone algorithm is equal to 
\begin{equation}
    C = \frac{1}{k}\sum_{a=1}^k \mathbb{P}(i_a \in S_\mathcal{A}), \label{eq:C_as_sum_prob1}
\end{equation}
where $i_a$ is the index of the $a^{th}$ secretary picked by the optimal offline solution ---i.e. $i_a$ is a top-$k$ secretary of $\mathcal{D}$. By Lemma \ref{lemma_monotone} \algoname \ is a monotone algorithm and we may use ~\eqref{eq:C_as_sum_prob1}.
\begin{align}
    \mathbb{P}(i_a \in S_\mathcal{A}) 
    &= \sum_{j=t}^{n-1} \mathbb{P}(i_a \in S_\mathcal{A} \text{ at time-step }j+1) \label{p_picked_equal_not_filled} \\
    &= \frac{1}{n}\sum_{j=t}^{n-1} \mathbb{P}(|S_{\mathcal{A}}| < k \text{ at time-step } j + 1) \notag
\end{align}
Now, we compute $\mathbb{P}(|S_{\mathcal{A}}| < k \text{ at time-step } j+1)$ by decomposing this probability into smaller events  $\mathbb{P}(|S_{\mathcal{A}}| = \nu \text{ at time-step } j+1)$ where $\nu \in [0,\dots,k-1]$.

We may compute the probability of $\mathbb{P}(|S_{\mathcal{A}}| = \nu \text{ at time-step } j+1)$ in the following manner. First, let us consider the scenario where $\nu$ elements are selected by \algoname\ at time steps $p_1$, $p_2, \dots, p_{\nu}$. 
Now, in order for an element to be selected at position $p_\nu$ that element must be one of the top $k$ elements up to time-step $j+1$. Therefore we have a factor $k/j$ in our equation. Now, in order to guarantee that no elements are picked after the position $p_\nu$ we additionally need to ensure that the remaining top-$k$ up to $j+1$ elements appear before $p_\nu$ which results in a factor of ${p_{\nu} - 1 \choose k - 1}/{j-1 \choose k - 1}$. 
Similarly, we may recursively calculate the corresponding factor for each position $p_{\nu-1} \dots p_1$. However, we also need to guarantee that no elements are picked within the time interval $[t+1 \dots p_1-1]$
---i.e. before $p_1$. The probability for this occurring is then ${t \choose k}/{p_1-1 \choose k}$ as this corresponds an ordering where the top-$k$ elements up to $p_1 - 1$ all appear in the sampling phase.
Thus, the probability $p_{t,j}^{k,\nu}:=\mathbb{P}(|S_{\mathcal{A}}| = \nu \text{ at time-step } j+1)$ is :

\begin{align}
    p_{t,j}^{k,\nu}
    &= \sum_{t+1 \leq p_1 < p_2 < \dots <p_{k-1} \leq j}\frac{k}{j} \frac{{p_{\nu} - 1 \choose k-1}}{{j-1 \choose k-1}}\frac{k}{p_{\nu} - 1}
    \frac{{p_{{\nu}-1} - 1 \choose k-1}}{{p_{\nu} - 2 \choose k-1}}\frac{k}{p_{\nu-1} - 1}\dots \frac{k}{p_2 - 1}\frac{{p_1 - 1 \choose k - 1}}{{p_2 - 2 \choose k-1}} \frac{{t \choose k}}{{p_1-1 \choose k}}\\
    & = \frac{t(t-1)\dots(t-k+1)}{j(j-1)\dots (j - k + 1)}\sum_{t+1 \leq p_1 < p_2 < \dots <p_{\nu} \leq j}
    \frac{k^{\nu}}{(p_\nu-k)(p_{\nu-1} - k)\dots(p_1 - k)}
\end{align}

Therefore, the probability of not exceeding $k$-selections, $p_{t,j}^k =\sum_{\nu=0}^{k-1} p_{t,j}^{k,\nu} $, to get before time step $j + 1$ is:
\begin{align}
   \notag
    p_{t,j}^k & = \frac{t(t - 1)...(t - k + 1)}{j (j - 1) \dots (j - k + 1)} \bigg( 1 + k\hspace{-1em}\sum_{p_1 = t + 1\dots j}^{} \Lambda_{p_1} + \dots
    + {k^{k - 1}}\hspace{-2em}\sum_{\substack{p_1 = t + 1 \dots p_2 - 1  \\\vdots\\ p_{k-1} = t + 1 \dots j}} \Lambda_{p_1} \dots \Lambda_{p_{k-1}} \bigg),
\end{align}

where we define $\Lambda_{p_i} := \frac{1}{p_i - k} $. The total competitive ratio is then: 
\cut{
\begin{align}
   \notag
    & \frac{t(t - 1)...(t - k + 1)}{j (j - 1) \dots (j - k + 1)} \bigg(1 + k\hspace{-1em}\sum_{p_1 = t + 1\dots j}^{}\frac{1}{p_1 - k} + \dots
    + {k^{k - 1}}\hspace{-2em}\sum_{\substack{p_1 = t + 1 \dots p_2 - 1  \\\vdots\\ p_{k-1} = t + 1 \dots j}}\frac{1}{(p_1 - k)\dots (p_{k-1} - k)} \bigg)
\end{align}
}

\begin{align}
    \label{app_general_k_eqn_23}
   C_k &= \frac{1}{n}\sum_{j = t}^{n - 1}
     \frac{t(t - 1)\dots (t - k + 1)}{j(j - 1)\dots(j - k + 1)} \bigg( 1 + k\hspace{-1em}\sum_{p_1 = t + 1\dots j}^{} \Lambda_{p_1} + \dots
    + {k^{k - 1}}\hspace{-2em}\sum_{\substack{p_1 = t + 1 \dots p_2 - 1  \\\vdots\\ p_{k-1} = t + 1 \dots j}} \Lambda_{p_1} \dots \Lambda_{p_{k-1}} \bigg), 
\end{align}

\cut{
\begin{align}
    & \frac{1}{n}\sum_{j = t}^{n - 1}
     \frac{t(t - 1)\dots (t - k + 1)}{j(j - 1)\dots(j - k + 1)} \bigg(1 + k\hspace{-1em}\sum_{p_1 = t + 1\dots j}^{}\frac{1}{p_1 - k} + \dots
    + {k^{k - 1}}\hspace{-2em}\sum_{\substack{p_1 = t + 1 \dots p_2 - 1  \\\vdots\\ p_{k-1} = t + 1 \dots j}}\frac{1}{(p_1 - k)\dots (p_{k-1} - k)} \bigg)
\end{align}
}

Now using Lemma ~\ref{lemma_three} we can bound it:
\begin{align}
     C_k & \geq
     \frac{1}{n}\int_{j = t}^{n}
     \frac{t(t - 1)\dots (t - k + 1)}{j(j - 1)\dots(j - k + 1)}\bigg(1 + \frac{k}{1!}\ln \Big(\frac{j - k}{t}\Big) + \dots + \frac{k^{k - 1}}{(k-1)!}\ln^{k - 1}\Big(\frac{j - k}{t}\Big)\bigg)\\
     & \geq
     \frac{1}{n}\int_{j = t}^{n}
     \frac{t(t - 1)\dots (t - k + 1)}{j^k}\bigg(1 + \frac{k}{1!}\ln \Big(\frac{j - k}{t}\Big) + \dots + \frac{k^{k - 1}}{(k-1)!}\ln^{k - 1}\Big(\frac{j - k}{t}\Big)\bigg)
\end{align}
Now notice that:
\begin{equation}
    \label{identity_two}
    \int\frac{1}{a!} \frac{\ln^a(x)}{x^k} dx = - \frac{1}{x^{k - 1}}\sum_{m = 0}^{a}\frac{1}{m!}(k -1)^{m - 1 -a }\ln^m(x)
\end{equation}
Using the identity in ~\eqref{identity_two} we compute the competitive ratio as:
\begin{align}
     & \geq\frac{t(t - 1)\dots (t - k + 1)}{n}
     \bigg(\sum_{a = 0}^{k - 1}- \frac{1}{j^{k - 1}}{ k ^ a }\sum_{m = 0}^{a}\frac{1}{m!}(k -1)^{m - 1 -a }\ln^m\Big(\frac{j - k}{t}\Big) \bigg) \Big|_{j = t}^{n}\\
     & =
     \frac{t(t - 1)\dots (t - k + 1)}{n}\bigg(- \frac{1}{j^{k - 1}}\sum_{m = 0}^{k - 1}\frac{1}{m!}\Big(\sum_{a = m}^{k - 1}{k ^ a}(k - 1)^{m - a - 1}\Big)\ln^m{\Big(\frac{j - k}{t}\Big)}\bigg)\Big|_{j = t}^{n}
\end{align}
For threshold $t = \alpha n $ where $\alpha \in (0,1)$ and as $n \xrightarrow{}\infty$ our competitive rate becomes:
\begin{align}
    & \alpha \bigg(\sum_{a = 0}^{k - 1}{ k ^ a }{(k - 1)}^{-1 - a}\Big) - \alpha^k\Big(\sum_{m=0}^{k - 1}\frac{1}{m!}\Big(\sum_{a = m}^{k - 1}{ k ^ a }(k - 1)^{m - a - 1}\Big)\ln^m\Big(\frac{1}{\alpha}\Big)\bigg)\\
    & = \alpha\left({\left(\frac{k}{k-1}\right)}^{k} - 1\right) - \alpha^k\left(\sum_{m-0}^{k-1}\left(\frac{\frac{k^k}{(k-1)^{k-m}} - k^m}{m!}\right)(-1)^{m+1}\ln^m(\alpha)\right)
\end{align}

Finally let us show that 
\begin{equation}
    f(\alpha) :=   {\alpha}^k \left(\sum_{m = 0}^{k - 1} a_m \ln^m (\alpha)\right) - \alpha a_0
    \quad where \quad  
    a_m = \left(\frac{\frac{k^k}{(k-1)^{k-m}} - k^m}{m!}\right)(-1)^{m+1}
\end{equation}
is concave. To do so we just compute it second derivative and show that $f''(\alpha) \leq 0$. We have,
\begin{align*}
    f''(\alpha) 
    &= \alpha^{k-2}\sum_{m = 0}^{k - 3} \left[ k(k-1)a_m + (2k-1) (m+1) a_{m+1} + (m+1)(m+2)a_{m+2} \right] \ln^{m-2} (\alpha) \\
    & \quad + [k(k-1)a_{k-2} + (2k-1) (k-1) a_{k-1}] \ln^{k-2}(\alpha) + k(k-1)a_{k-1} \ln^{k-1}(\alpha) \,.
\end{align*}
By using the definition of $a_m$, we can verify that 
\begin{align*}
     &k(k-1)a_m + (2k-1) (m+1) a_{m+1} + (m+1)(m+2)a_{m+2} =  0\\
     &\text{and} \quad k(k-1)a_{k-2} + (2k-1) (k-1) a_{k-1} = 0 \,.
\end{align*}
Thus we finally get, 
\begin{equation}
    f''(\alpha) =  k(k-1)a_{k-1} \ln^{k-1}(\alpha) = -\frac{k^2 \big(\alpha k \log(\frac1\alpha)\big)^{k-1}}{\alpha k!} \leq 0
\end{equation}
where we use the fact that since $\alpha \in [0,1]$, we have $\alpha \log(1/\alpha) \geq 0$.
\end{proof}
\begin{definition}
\label{monotone_defn}
An algorithm is called monotone if the probabilities of selecting items $i$ and $j$ satisfy $p_i \geq p_j$ whenever the item values $v_i > v_j$ holds for any two items.
\end{definition}

\begin{lemma}
\label{lemma_monotone}
\algoname \ is a monotone algorithm. 
\end{lemma}
\begin{proof}
    In order to prove that \algoname \ is monotone as defined in Definition \ref{monotone_defn} we must prove that $p_i \geq p_j$ (where $p_i$ is the probability of picking the item $i$) for any two items where $v_i > v_j$. Without loss of generality let us consider a decreasing ordering of $n$-elements based on their values ---i.e. $v_1 > v_{2} > \dots > v_n$.

    We prove that $p_i \geq p_{i+1}$ for all $i \in [1, \dots, n-1]$ by showing that for each input sequence where $v_{i+1}$ is accepted, there exists a unique input sequence where $v_i$ is accepted. Let us consider a permutation $\pi$ where $v_{i+1}$ appeared and was accepted at time step $a$ while $v_i$ appeared at time step $b$. By swapping $v_i$ and $v_{i+1}$ we obtain a new permutation $\pi'$ where $v_i$ now appears at $a$ and $v_{i+1}$ at $b$. We now study the two following cases.
    
    \xhdr{Case 1: $a < b$}
    
    If $a < b$ notice that the reference set, $R$, and the selected set $S_{\mathcal{A}}$, are exactly the same at time step $a$ for both permutations $\pi$ and $\pi'$. Therefore, if $v_{i+1}$ was accepted at time step $a$ in permutation $\pi$ then $v_i$ will also be accepted at time step $a$ in permutation $\pi'$ since $v_i > v_{i+1}$.   

    \xhdr{Case 2: $a > b$}
    
    If $a > b$ notice that $R$---by definition of \algoname---at time step $a$ contains top-$k$ elements observed in the first $a-1$ time steps. 
    Now the $k$-th element in $R$ at time-step $a$ must satisfy,
    \begin{equation*}
        R^a_{\pi}[k] \geq R^a_{\pi'}[k],
    \end{equation*}
    where $R^a_{[\cdot]}[k]$ corresponds to the $k$-element in the reference set for a specific permutation at time step $a$. Hence, we know that $v_{i} > v_{i+1} \geq R^a_{\pi}[k] \geq R^a_{\pi'}[k]$ as $v_{i+1}$ was assumed to be picked.

    Furthermore, the $S_{\mathcal{A}}$ and $R$ is the same for permutations $\pi$ and $\pi'$ at time-step $b$. Now by our primary assumption that $v_{i+1}$ is picked at time-step $a>b$ in $\pi$ this means that $v_i$ must be  $v_i \geq R^b_{\pi}[k]$ since $v_i > v_{i+1}$. However, observe that $v_i$ and $R^b_{\pi}[k]$ cannot be consecutive in value as $v_{i+1}$ appears at time-step $a > b$ in permutation $\pi$. This implies that $v_{i+1}$ must also be selected at time step $b$ in permutation $\pi'$ since $v_i$ and $v_{i+1}$ are consecutive in value. By a similar argument based on consecutive order of values between time steps $a$ and $b$ precisely the same elements will be selected in both $\pi$ and $\pi'$. The argument that   $v_{i} > v_{i+1} \geq R^a_{\pi}[k] \geq R^a_{\pi'}[k]$ implies that if $v_{i+1}$ is selected in permutation $\pi$, $v_i$ will also be selected in permutation $\pi'$. The claim then follows by applying the inequality $p_i \geq p_{i+1}$ in an iterative fashion.
\end{proof}
\begin{lemma} 
\label{lemma_three}
Let $f_i \,, \,i = 1 \ldots k$ be decreasing positive functions then we have 
\begin{equation}
    \sum_{p_1=a_1}^{b_1} \ldots \sum_{p_k=a_k}^{p_{k-1}} f_1(p_1) \ldots f_k(p_k)
    \geq \int_{x_1=a_1}^{b_1+1} \ldots \int_{x_k = a_k}^{x_{k-1}+1}  f_1(x_1) \ldots f_k(x_k) dx_1\dots dx_k
\end{equation}
\end{lemma}
\begin{proof}
    The main proof step involves in first noticing that since the functions $f_i \,, \,i = 1 \ldots k$ are decreasing and are positive we have, 
    \begin{equation}
        f_1(p_1) \ldots f_k(p_k) \geq  f_1(p_1) \ldots  f_{k-1}(p_{k-1})\int_{x_k = p_k}^{p_k+1} f_k(x_k) dx_k
    \end{equation}
    Thus, by summing this inequality for $p_k= a_k \ldots p_{k-1}$, we get
    \begin{equation}
    \sum_{p_k=a_k}^{p_{k-1}} f_1(p_1) \ldots f_k(p_k)
    \geq  f_1(p_1) \ldots  f_{k-1}(p_{k-1})  \int_{x_k = a_k}^{p_{k-1}+1} f_k(x_k) dx_k
    \end{equation}
    Now, because the functions $f_i \,, \,i = 1 \ldots k$ are decreasing and positive we have, 
     \begin{align}
    \mathcal{S} 
    &= \sum_{p_k=a_k}^{p_{k-1}} f_1(p_1) \ldots f_k(p_k)\\
    &\geq  f_1(p_1) \ldots  f_{k-2}(p_{k-2})  \int_{x_{k-1} = p_{k-1}}^{p_{k-1}+1} f_{k-1}(x_{k-1})  \int_{x_k = a_k}^{p_{k-1}+1} f_k(x_k) dx_{k-1}dx_k \\
    &\geq f_1(p_1) \ldots  f_{k-2}(p_{k-2})  \int_{x_{k-1} = p_{k-1}}^{p_{k-1}+1} f_{k-1}(x_{k-1})  \int_{x_k = a_k}^{x_{k-1}} f_k(x_k) dx_{k-1}dx_k 
    \end{align}
    where for the last inequality we used the fact that $x_{k-1} \in [p_{k-1}, p_{k-1}+1]$.
    Finally, by summing for $p_{k-1} = a_{k-1} \ldots p_{k-2}$, we get,
    \begin{align}
    \sum_{p_{k-1}=a_{k-1}}^{p_{k-2}}\mathcal{S} &= \sum_{p_{k-1}=a_{k-1}}^{p_{k-2}} \sum_{p_k=a_k}^{p_{k-1}} f_1(p_1) \ldots f_k(p_k)\\
    &\geq f_1(p_1) \ldots  f_{k-2}(p_{k-2})  \int_{x_{k-1} = p_{k-1}}^{p_{k-1}+1} f_{k-1}(x_{k-1})  \int_{x_k = a_k}^{x_{k-1}} f_k(x_k) dx_{k-1}dx_k
    \end{align}
    Using a recursive argument we finally get,
    \begin{equation}
    \sum_{p_1=a_1}^{b_1} \ldots \sum_{p_k=a_k}^{p_{k-1}} f_1(p_1) \ldots f_k(p_k)
    \geq \int_{x_1=a_1}^{b_1+1} \ldots \int_{x_k = a_k}^{x_{k-1}+1}  f_1(x_1) \ldots f_k(x_k) dx_1\dots dx_k
\end{equation}
\end{proof}

\clearpage
\subsection{Analytic computation of $C_k$ for \algoname}

\begin{table}[ht!]
    \centering
    \caption{Values of the Competitive ratio $C_k$ and the associated optimal $\alpha_k$ needed to compute the threshold for \algoname. Note that for $5\leq k\leq 100$ the competitive ratio of \textsc{Single-Ref} provided by~\citet{albers2020new} outperforms \algoname's competitive ratio. However, our analysis provides a tractable way to scale the analytic computation of the competitive ratio with $k$ as the function to optimize (and its gradients) in Theorem~\ref{thm:general_k_theorem1} is $\mathcal{O}(k)$.}
    \vspace{2pt}
    \begin{tabular}{ccccccccccc}
     \toprule
          $k$ & 2 &3 & 4 & 5 & 100 & 200 &300 &400 & 500 & 600\\
         \midrule
        $C_k$ & .4273&.4575&.4769&.4906&.5959&.6062&.6108&.6136&.6156&.6170 \\
        $\alpha_k$ & .3824 &.3867&.3884&.3890&.3781 &.3755 &.3743&.3735&.3729&.3726 \\
        \bottomrule
    \end{tabular}
    \label{tab:C_k}
\end{table}

\section{Proof of Theorem~\ref{thm:stochastic_secretary} and empirical verification}
\label{appendix:proof_thm2}
We now prove Theorem \ref{thm:stochastic_secretary} in detail, reproduced here for convenience. 

\begin{reptheorem}{thm:stochastic_secretary}
Let us assume that \algoname\ observes independent random variables $\mathcal{V}_{i}$ following Assumption~\ref{assump:feas}. Its stochastic competitive ratio $C_{s}$ can be bounded as follows,
\begin{equation}
  C \geq C_{s} \geq \gamma C 
\end{equation}
\end{reptheorem}

\begin{proof}
In the Stochastic case we use the same beginning proof as in the non-stochastic case until~\eqref{p_picked_equal_not_filled}. Let us consider $i_a$ such that $\mathcal{V}_{i_a} \in \topk\{\mathcal{V}_i\}$,
\begin{align}
    \mathbb{P}(i_a \in S_\mathcal{A}) 
    &= \sum_{j=t}^{n-1} \mathbb{P}(i_a \in S_\mathcal{A} \text{ at time-step }j+1) 
\end{align}
Now, in the stochastic case $\mathbb{P}(i_a \in S_\mathcal{A} \text{ at time-step }j+1)$ not only depends on the set $\mathcal{S}_\mathcal{A}$ not being full but also that $i_a$ corresponds to a top-$k$ elements in $\{v_i\}$. Because of the expectation over permutations, the event of the knapsack being fulled at time step $j+1$ is independent from what happens at timestep $j+1$. Thus we can write that  
\begin{align*}
     \mathbb{P}(i_a \in S_\mathcal{A} \text{ at time-step }j+1)  
     &= \mathbb{P}( |S_\mathcal{A}|<k \text{ at time-step }j+1)
     \sP[\mathcal{V}_{i} \in \topk\{\mathcal{V}_i\}\,|\,v_i \in \topk\{v_i\}] \\
     &\geq  \mathbb{P}( |S_\mathcal{A}|<k \text{ at time-step }j+1) \gamma \,.
\end{align*}
Finally, we just need to notice that $\sP[\mathcal{V}_{i} \in \topk\{\mathcal{V}_i\}\,|\,v_i \in \topk\{v_i\}]$ does not depend on the value observed and thus leave to the same computation as in the non-stochastic case.

Similarly, we have 
\begin{align*}
     \mathbb{P}(i_a \in S_\mathcal{A} \text{ at time-step }j+1)  
     &= \mathbb{P}( |S_\mathcal{A}|<k \text{ at time-step }j+1)
     \sP[\mathcal{V}_{i} \in \topk\{\mathcal{V}_i\}\,|\,v_i \in \topk\{v_i\}] \\
     &\leq  \mathbb{P}( |S_\mathcal{A}|<k \text{ at time-step }j+1) \,.
\end{align*}

In conclusion it leads to 
\begin{equation}
  C \geq C_{s} \geq \gamma C 
\end{equation}

\end{proof}

\subsection{Empirical Quantification of Assumption 1}
\label{app:empirical_quantification_assumption_one}
Assumption \ref{assump:feas} requires that a percentage of top-$k$ true values $v_i$ remain top-$k$ under noise. We now empirically quantify the strength of this assumption in both of our datasets MNIST and CIFAR-10. Note that unlike in theorem we cannot enforce any structure on the random variables that act as surrogate losses as provided by $f_s$. Despite this, we find that in all cases the overlap between the top-$k$ sets is non-zero which enables the effective use of online algorithms for picking candidate attack points as shown in table \ref{app:table_with_overlap}.

\begin{table*}[ht]
\footnotesize
\caption{Number of top-$k$ elements in $\{v_i\}$ that are also top-$k$ elements in $\{\mathcal{V}_{i}\}$ for the different setting considered in the paper. $\vert \topk\{\mathcal{V}_i\} \cap \topk\{v_i\} \vert$ Note that $n=10000$.}
 \begin{center}\begin{tabular}{ c c c c | c c c }
 \toprule
 & \multicolumn{3}{c}{MNIST} & \multicolumn{3}{c}{CIFAR-10}\\
  & $k=10$ & $k=100$ & $k=1000$ & $k=10$ & $k=100$ & $k=1000$ \\
 \midrule
 FGSM & 0.9 $\pm$ 0.1 & 16.1 $\pm$ 0.7 & 324.0 $\pm$ 7.0 & 0.60 $\pm$ 0.03 & 12.5 $\pm$ 0.2 & 333.2 $\pm$ 2.8\\
 PGD & 0.58 $\pm$ 0.03 & 7.1 $\pm$ 0.3 & 229.9 $\pm$ 3.3 & 0.59 $\pm$ & 10.0 $\pm$ 0.3 & 227.1 $\pm$ 3.6\\
 \bottomrule
\label{app:table_with_overlap}
\end{tabular}\end{center} 
\end{table*}

\section{Classical Online Algorithms for Secretary Problems}
\label{appendix:classical_online_algorithms}

All single threshold online algorithm described in this paper include: \textsc{Virtual}, \textsc{Optimistic} and \textsc{Single-Ref}. Each online algorithm consists of two phases ---\textbf{sampling phase} followed by \textbf{selection phase}--- and an optimal stopping point $t$ which is used by the algorithm to transition between the phases. We now briefly summarize these two phases for the aforementioned online algorithms. 

\xhdr{Sampling Phase - \textsc{Virtual}, \textsc{Optimistic} and \textsc{Single-Ref}}
In the sampling phase, the algorithms passively observe all data points up to a pre-specified time index $t$, but also maintains a sorted reference list $R$ consisting of the $k$ elements with the largest values $\mathcal{V}(i)$ seen. Thus the $R$ contains a list of elements sorted by decreasing value. That is $R[k]$ is the index of the $k$-th largest element in $R$ and $\mathcal{V}(R[k])$ is its corresponding value. The elements in $R$ are kept for comparison but are crucially \textit{not} selected in the sampling phase.

\subsection{\textsc{Virtual} Algorithm}

\xhdr{Selection Phase - \textsc{Virtual} algorithm}
Subsequently, in the selection phase, $i > t$, when an item with value $\mathcal{V}(i)$ is observed an irrevocable decision is made of whether the algorithm should select $i$ into $S$. To do so, the Virtual algorithm simply checks if the value of the $k$-th smallest element in $R$, $\mathcal{V}(R[k])$, is smaller than $\mathcal{V}(i)$ in addition to possibly updating the set $R$. The full Virtual algorithm is presented in Algorithm 1.

\begin{algorithm}[ht]
\textbf{Inputs:} $t\in[k\dots n-k]$, $R = \emptyset$, $S_{\mathcal{A}} = \emptyset$
\newline
\textbf{Sampling phase:} Observe the first $t$ data points and construct a list $R$ with the indices of the top $k$ data points seen.  $\texttt{sort}$ ensures: $ \mathcal{V}(R[1]) \geq \mathcal{V}(R[2]) \dots \geq \mathcal{V}(R[k]).$

\textbf{Selection phase (at time $i>t$):}

\begin{algorithmic}[1]
\IF {$ \mathcal{V}(i) \geq  \mathcal{V}(R[k])$
and $R[k] > t$} 
        \STATE $R$ = $\texttt{sort}\{R \cup \{i\} \setminus \{R[k]\}\}$ \hfill\COMMENT{// Update $R$ with element $i$ and also take out $R[k]$}
\ELSIF{ $ \mathcal{V}(i) \geq  \mathcal{V}(R[k])$
and $R[k] \leq t$}
        \STATE $R$ = $\texttt{sort}\{R \cup \{i\} \setminus \{R[k]\}\}$ \hfill\COMMENT{// Update $R$ with element $i$ and also take out $R[k]$}
        \STATE $S_\mathcal{A} = \{ S_\mathcal{A} \cup \{i \}\}$ \hfill\COMMENT{// Select element $i$}
\ENDIF 
\STATE
$i\gets i + 1$
\end{algorithmic}
 \caption{\textsc{Virtual Algorithm}}
\end{algorithm}

\subsection{\textsc{Optimistic} Algorithm}

\xhdr{Selection Phase - \textsc{Optimistic} algorithm}
In the optimistic algorithm, i is selected if and only if $\mathcal{V}(i) \geq \mathcal{V}(R[last])$. Whenever $i$ is selected, $R[last]$
is removed from the list
$R$, but no new elements are ever added to $R$. Thus, intuitively, elements are selected when they beat one of the remaining reference
points from $R$.
We call this algorithm “optimistic” because it removes the reference
point $R[last]$ even if $ \mathcal{V}(i)$ exceeds, say, $ \mathcal{V}(R[1])$. Thus, it implicitly assumes
that it will see additional very valuable elements in the future, which
will be added when their values exceed those of the remaining, more
valuable, $R[a]\,,\,a \in [k]$.

\begin{algorithm}[ht]
\textbf{Inputs:} $t\in[k\dots n-k]$, $R = \emptyset$, $S_{\mathcal{A}} = \emptyset$.

\textbf{Sampling phase (up to time $t$):} Observe the first $t$ data points and construct a list $R$ with the indices of the top $k$ data points seen. $\texttt{sort}$ ensures: $ \mathcal{V}(R[1]) \geq \mathcal{V}(R[2]) \dots \geq\mathcal{V}(R[k]).$
Set $last = k$, to be the index of the last element in $R$.

\textbf{Selection phase (at time $i>t$):}

\begin{algorithmic}[1]
\IF {$\mathcal{V}(i) \geq \mathcal{V}(R[last])$} 
        \STATE $R$ = $\{R \setminus \{R[last]\}\}$ \hfill\COMMENT{// Update $R$ by taking out $R[k]$}
        \STATE $S_\mathcal{A} = \{ S_\mathcal{A} \cup \{i \}\}$ \hfill\COMMENT{// Select element $i$}
        \STATE $last = last - 1$
\ENDIF 
$i\gets i + 1$
\end{algorithmic}
 \caption{\textsc{Optimistic Algorithm}}
\end{algorithm}

\subsection{\textsc{Single-Ref} Algorithm}
\xhdr{Selection Phase - \textsc{Single-Ref} algorithm}
In the \textsc{Single-Ref} algorithm, $i$ is selected if and only if $\mathcal{V}(i) \geq \mathcal{V}(R[r])$ and we haven't already selected $k$ elements. 
We call this algorithm single reference algorithm because we always compare incoming elements to one single reference element, that was determined in the sampling phase.

\begin{algorithm}[ht]
\textbf{Inputs:} $t\in[k\dots n-k]$, $R = \emptyset$, $S_{\mathcal{A}} = \emptyset$, $r \in [k]$ (reference rank)

 \textbf{Sampling phase (up to time $t$):} Observe the first $t$ data points and construct a list $R$ with the indices of the top $k$ data points seen. Let $s_r = R[r]$ be the $r$-th best item from the sampling phase.

\textbf{Selection phase (at time $i>t$):} 
\begin{algorithmic}[1]
\IF {$\mathcal{V}(i) \geq s_r$ and $ |S_{\mathcal{A}}| \leq  k$} 
        \STATE $S_\mathcal{A} = \{ S_\mathcal{A} \cup \{i \}\}$ \hfill\COMMENT{// Choose the first $k$ items better than $s_r$}

\ENDIF 

$i\gets i + 1$
\end{algorithmic}
 \caption{\textsc{Single-Ref Algorithm}}
\end{algorithm}

\section{Additional Experimental Results}
\label{appendix:additional_results}

In this appendix we provide detailed results on all experiments against non-robust models. For MNIST and CIFAR-10 we compute the average online fool rate over $1000$ runs while for Imagenet we used $5$ runs due to the increased computational resources required. Our pretrained models for MNIST and CIFAR-10 can be found at the footnote link below \footnote{\tiny \texttt{https://drive.google.com/drive/folders/1RLjWmkmZ5DC\_0sFfpqCdZH2zcG7lgWcH?usp=sharing}}

\subsection{Detailed results on Non-robust results}
\label{appendix:detailed_non_robust_results}

\begin{table*}[ht]
\vspace{-10pt}
\footnotesize
\caption{Online fool rate of various online algorithms on non-robust models. For a given attack and value of $k$: {\color{g1} $\mathbf{\bullet}$ } at least 97\%,
\textbf{\color{g2} $\mathbf{\bullet}$} at least 95\%, \textbf{\color{g3}$\mathbf{\bullet}$} at least 90\%, \textbf{\color{g4} $\mathbf{\bullet}$} less than 90\% of the optimal performance.}
\vspace{-5pt}
\label{table:non_robust_table1}
 \begin{center}\begin{tabular}{ c c c c c c c c }
 \toprule
 & & \multicolumn{3}{c}{MNIST (Online fool rate in \%)} & \multicolumn{3}{c}{CIFAR-10  (Online fool rate in \%)}\\
 & Algorithm & $k=10$ & $k=100$ & $k=1000$ & $k=10$ & $k=100$ & $k=1000$ \\
 \midrule
 \multirow{6}{*}{\rotatebox[origin=c]{90}{FGSM}}
 & \textsc{Naive}& 64.1 $\pm$ 33 & 47.8 $\pm$ 31 & 45.7 $\pm$ 31 & 60.7 $\pm$ 17 & 59.2 $\pm$ 6.2 & 59.2 $\pm$ 4.3\\
 & \textsc{Opt}  & \textbf{87.0 $\pm$ 0.5} & \textbf{84.7 $\pm$ 0.5} &  \textbf{83.6 $\pm$ 0.4} & \textbf{86.6 $\pm$ 0.4} & \textbf{87.3 $\pm$ 0.3} &  \textbf{86.5 $\pm$ 0.2} \\
 \cmidrule{2-8}
 & \textsc{Optimistic} & \cellcolor{g3}79.0 $\pm$ 0.5 & \cellcolor{g3}77.6 $\pm$ 0.4 &\cellcolor{g3} 75.3 $\pm$ 0.4 &\cellcolor{g4} 75.3 $\pm$ 0.5 & \cellcolor{g4} 72.8 $\pm$ 0.2 &\cellcolor{g4} 71.9 $\pm$ 0.2\\
 & \textsc{Virtual} & \cellcolor{g3}78.6 $\pm$ 0.5 &\cellcolor{g3} 79.1 $\pm$ 0.4 &\cellcolor{g3} 77.4 $\pm$ 0.4 & \cellcolor{g4} 76.1 $\pm$ 0.5 &\cellcolor{g4} 77.1 $\pm$ 0.2 & \cellcolor{g4}75.4 $\pm$ 0.2\\
 & \textsc{Single-Ref} &\cellcolor{g2}85.1 $\pm$ 0.5 & \cellcolor{g1}83.0 $\pm$ 0.5 &\cellcolor{g4} 72.3 $\pm$ 0.5 &\cellcolor{g3} 80.4 $\pm$ 0.5 &\cellcolor{g2} 84.0 $\pm$ 0.3 & \cellcolor{g4}66.0 $\pm$ 0.2\\
 & \algoname & \cellcolor{g3}80.4 $\pm$ 0.5 &\cellcolor{g1} 82.5 $\pm$ 0.4 & \cellcolor{g1} 82.9 $\pm$ 0.4 &\cellcolor{g2} 82.9 $\pm$ 0.5 & \cellcolor{g1}86.3 $\pm$ 0.3 &  \cellcolor{g1}85.2 $\pm$ 0.2\\
 \midrule
  \multirow{6}{*}{\rotatebox[origin=c]{90}{PGD}}
 & \textsc{Naive} & 69.7  $\pm$ 16 & 67.2 $\pm$ 20 & 67.9 $\pm$ 18 & 72.5 $\pm$ 18 & 70.4 $\pm$ 9.4 & 68.6 $\pm$ 6.3\\
 & \textsc{Opt} & \textbf{73.6 $\pm$ 0.9} & \textbf{49.8 $\pm$ 0.8} & \textbf{49.6 $\pm$ 0.8} & \textbf{83.7 $\pm$ 0.6} & \textbf{80.6 $\pm$ 0.6} & \textbf{79.9 $\pm$ 0.5}\\
 \cmidrule{2-8}
 & \textsc{Optimistic} & \cellcolor{g3} 66.2 $\pm$ 1.1 & \cellcolor{g2} 48.2 $\pm$ 0.8 &\cellcolor{g3} 45.1 $\pm$ 0.9 &\cellcolor{g3} 79.1 $\pm$ 0.6 & \cellcolor{g3}76.6 $\pm$ 0.4 & \cellcolor{g3}76.0 $\pm$ 0.4\\
 & \textsc{Virtual} &\cellcolor{g4} 63.4 $\pm$ 1.1 & \cellcolor{g4}46.2 $\pm$ 0.9 &\cellcolor{g2} 46.8 $\pm$ 0.8 & \cellcolor{g3}78.3 $\pm$ 0.6 &\cellcolor{g3} 77.5 $\pm$ 0.5 &\cellcolor{g3} 76.9 $\pm$ 0.4\\
 &\textsc{Single-Ref} & \cellcolor{g1} 71.5 $\pm$ 0.9 & \cellcolor{g1} 49.7 $\pm$ 0.8 & \cellcolor{g4}42.9 $\pm$ 0.9 & \cellcolor{g2}80.2 $\pm$ 0.6 &\cellcolor{g1} 79.6 $\pm$ 0.5 & \cellcolor{g4} 74.5 $\pm$ 0.4\\
 & \algoname         &\cellcolor{g2} 68.2 $\pm$ 1.0 & \cellcolor{g1}49.3 $\pm$ 0.8 & \cellcolor{g1} 49.7 $\pm$ 0.8 & \cellcolor{g1} 81.2 $\pm$ 0.6 & \cellcolor{g1} 80.1 $\pm$ 0.6 &\cellcolor{g1}79.5 $\pm$ 0.5\\
 \bottomrule
\end{tabular}\end{center} 
\vspace{-5mm}
\end{table*}

\begin{table*}[ht]
\footnotesize
\caption{Competitive ratio on non-robust models using FGSM and PGD attacker and various online algorithms on ImageNet.}
\label{appendix:comp_ratio_non_robust}
 \begin{center}\begin{tabular}{ c c c c c c c c }
 \toprule
 & & \multicolumn{3}{c}{Imagenet (Online Fool Rate in \%)} & \\
 & Algorithm & $k=10$ & $k=100$ & $k=1000$ &  \\
 \midrule
 \multirow{6}{*}{\rotatebox[origin=c]{90}{FGSM}}
 & \textsc{Naive} & 66.7 $\pm$ 7.7 & 66.3 $\pm$ 2.1 & 65.0 $\pm$ 2.2 \\
 & \textsc{Opt} & \textbf{98.7 $\pm$ 0.9} & \textbf{95.3 $\pm$ 1.8} & \textbf{96.2 $\pm$ 1.1} \\
 \cmidrule{2-5}
 & \textsc{Optimistic} &  \cellcolor{g4} 86.0 $\pm$ 2.8 &  \cellcolor{g4}80.4 $\pm$ 1.6 &  \cellcolor{g4} 79.9 $\pm$ 1.4 \\
 & \textsc{Virtual} &  \cellcolor{g4} 85.3 $\pm$ 2.6 &  \cellcolor{g4} 84.9 $\pm$ 1.7 &  \cellcolor{g4} 84.3 $\pm$ 1.2 \\
 & \textsc{Single-Ref} &  \cellcolor{g2} 94.0 $\pm$ 2.5 & \cellcolor{g2} 92.4 $\pm$ 1.9 &\cellcolor{g4}  72.5 $\pm$ 1.7 \\
 & \algoname &   \cellcolor{g1} 96.0 $\pm$ 1.6 & \cellcolor{g1} 95.0 $\pm$ 1.2 & \cellcolor{g1}  95.8 $\pm$ 1.0\\
 \midrule
 \multirow{6}{*}{\rotatebox[origin=c]{90}{PGD}}
 & \textsc{Naive} &72.5 $\pm$ 5.4  & 72.5 $\pm$ 3.8 & 73.8 $\pm$ 4.8 \\
 & \textsc{Opt} &  \textbf{82.5 $\pm$ 7.4} & \textbf{80.2 $\pm$ 4.7} & \textbf{76.8 $\pm$ 5.5}\\
 \cmidrule{2-5}
 & \textsc{Optimistic} &  \cellcolor{g1}87.5 $\pm$ 5.4 & \cellcolor{g1}78.0 $\pm$ 4.3 & \cellcolor{g1}74.5 $\pm$ 5.1 \\
 & \textsc{Virtual} & \cellcolor{g2}80.0 $\pm$ 9.4 & \cellcolor{g3}74.0 $\pm$ 4.8 & \cellcolor{g1}75.6 $\pm$ 5.1\\
 &\textsc{Single-Ref} & \cellcolor{g3}77.5 $\pm$ 5.4 & \cellcolor{g1}79.5 $\pm$ 5.9 & \cellcolor{g1}75.2 $\pm$ 5.1  \\
 & \algoname & \cellcolor{g3} 77.5 $\pm$ 8.2 & \cellcolor{g1}79.0 $\pm$ 6.6 & \cellcolor{g1}76.4 $\pm$ 5.4\\
 \bottomrule
\end{tabular}\end{center} 
\end{table*}

\subsection{Additional Results on Synthetic Data}
\label{appendix:synthetic_additional_results}
We now provide additional results on Synthetic Data with varying levels of noise added to each item in $\mathcal{D}$. In particular, we investigate in figure \ref{fig:additional_results_synthetic_data} online algorithms in the face of no noise ---i.e. $\sigma^2 = 0$, $\sigma^2 = 1$, and $\sigma^2 = 5$ in addition to $\sigma^2 = 10$ reported in figure \ref{fig:synthetic_data}. The deterministic setting corresponds to $\sigma^2 = 0$ while $\sigma^2=1$ and $\sigma^2=1$ correspond to the stochastic setting as introduced in section \ref{stochastic_k_secretary}.

\begin{figure}[ht]
    \centering
    \includegraphics[width=0.32\linewidth]{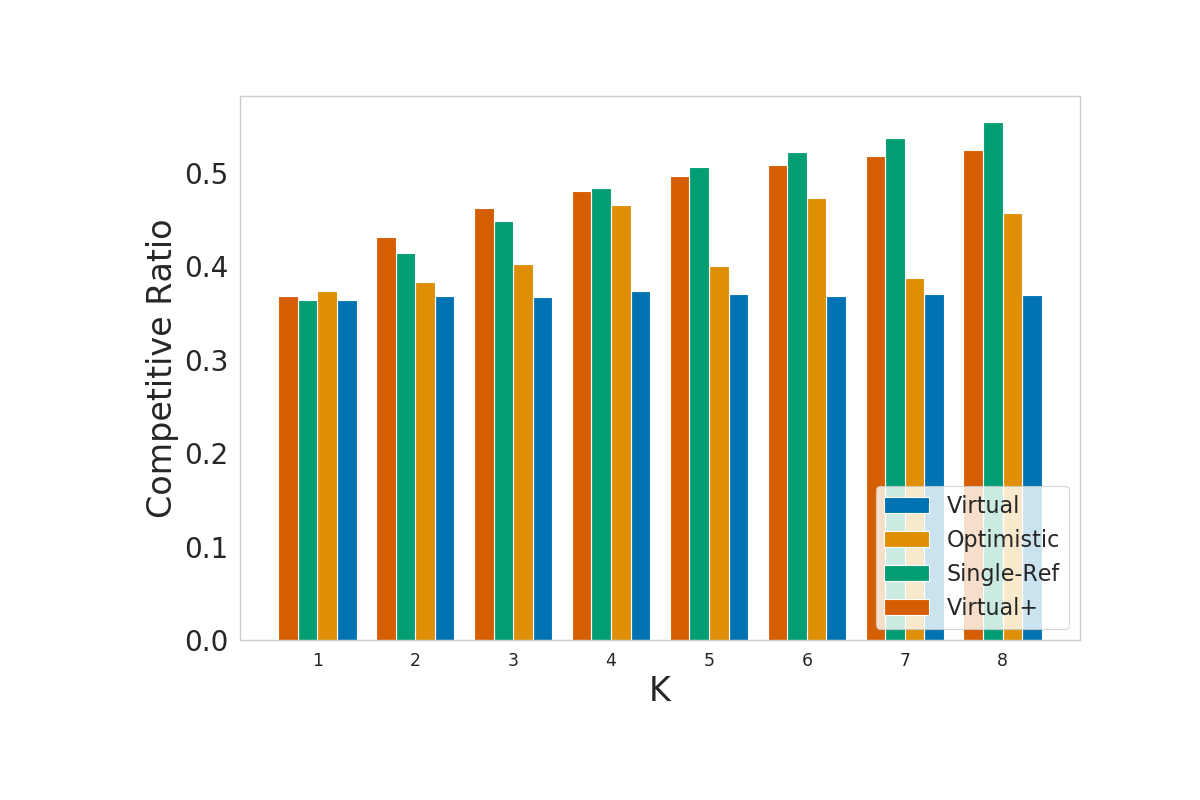}
    \includegraphics[width=0.32\linewidth]{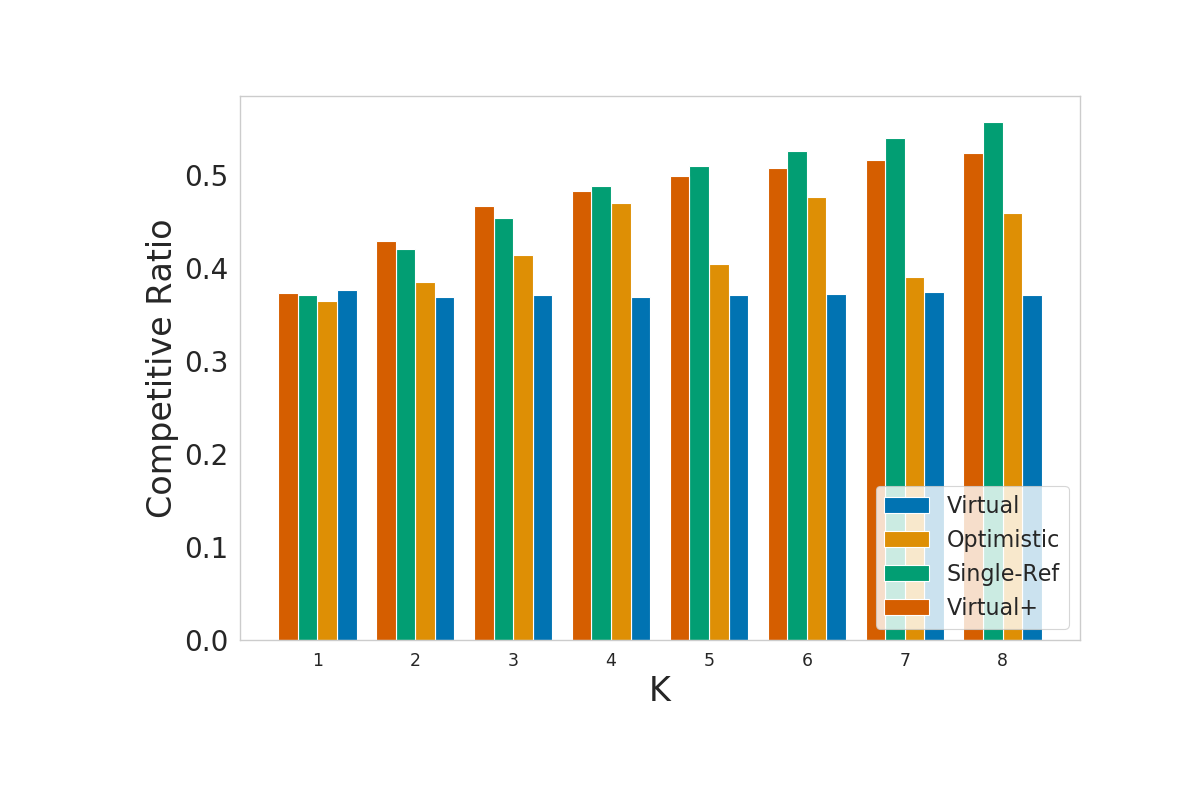}
    \includegraphics[width=0.32\linewidth]{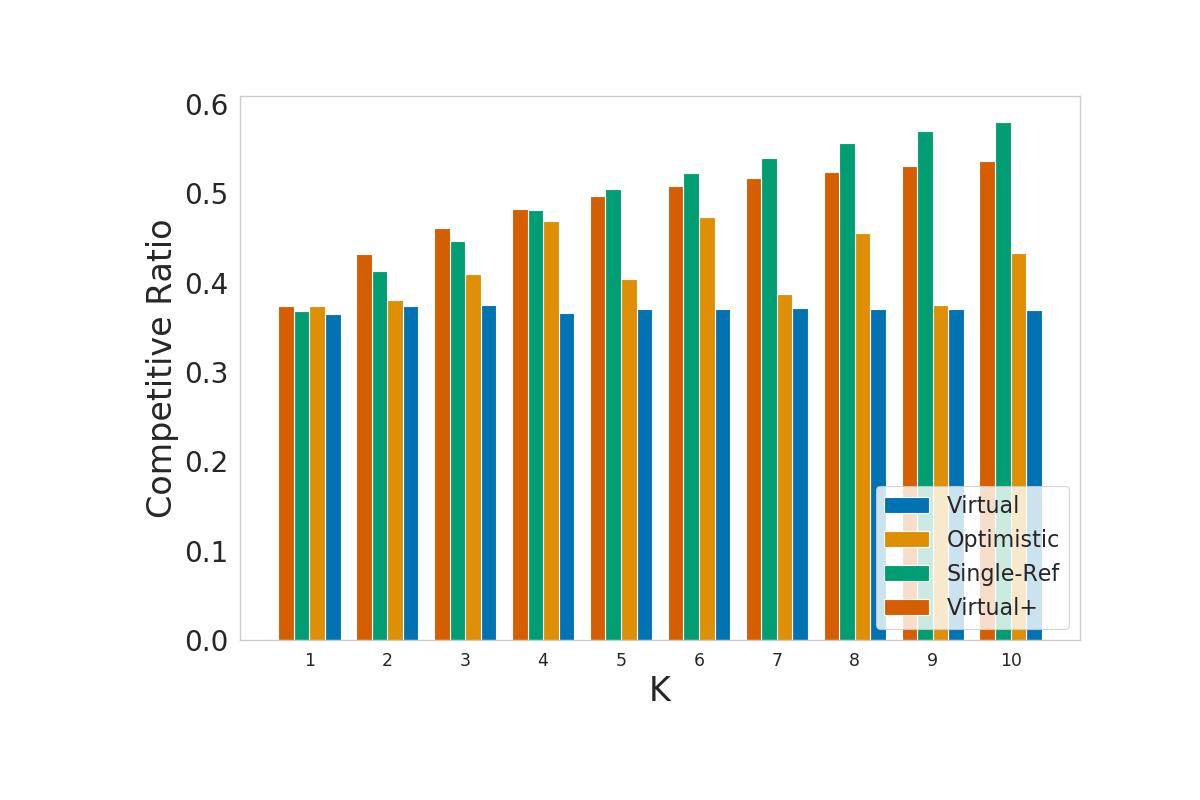}
    \caption{Estimation of the competitive ratio of online algorithms under various noise levels. \textbf{Left:} Deterministic setting with $\sigma^2=0$. \textbf{Middle:} Stochastic setting with $\sigma^2 = 1$. \textbf{Right:} Stochastic setting with $\sigma^2 = 5$.}
    \label{fig:additional_results_synthetic_data}
\end{figure}

\subsection{Experimental Details}
\label{appendix:additional_results_larger_datasets}
We provide more details about the experiments presented in section~\ref{section:experiments}. For further details we also invite the reader to look at the code provided with the supplementary materials. The complete code to reproduce results can be found \url{https://anonymous.4open.science/r/OnlineAttacks-4349} 

\paragraph{Attack strategies} We use two different attack strategies the Fast Gradient Sign Method (FGSM) \citep{goodfellow2014explaining} and 40 iterations of the PGD attack \citep{madry2017towards} with $l_\infty$.

\paragraph{Hyper-parameters of online algorithms} All the online algorithms except \textsc{Single-Ref} have a single hyper-parameters to choose which is the length of the sampling phase $t$. For \textsc{Virtual} and \textsc{Optimistic} we use $t=\floor{ \frac{t}{e}}$ which is the value suggested by theory in \citet{babaioff2007knapsack}. For \algoname \ we use $t=\alpha n$  as found by solving the maximization problem for a specific $k$ in Theorem~\ref{thm:general_k_theorem1}. \textsc{Single-Ref} has two hyper-parameters to choose the threshold $t$ ($c$ in the original paper) and reference rank $r$. For $k=1...100$ the values are given in \citet{albers2020new} and are numerical solutions to combinatorial optimization problems. However, for $k = 1000$ no values are specified and we choose $c=0.13$ and $r=40$ through grid search. Indeed, these values may not be optimal ones but we leave the choice of better values as future work.

\paragraph{MNIST model architectures} For table~\ref{table:non_robust_table1}, $f_s$ and $f_t$ are chosen randomly from an ensemble of trained classifiers. The ensemble is composed of five different architectures described in table~\ref{appendix:mnist_ens_adv_training_table}, with 5 trained models per architecture.

\begin{table}[h]
    \centering
          
            \begin{tabular}{ccccc}
            \toprule
                   &  A  & B  & C & D \\
                    \midrule
                   & Conv(64, 5, 5) + Relu & Dropout(0.2) & Conv(128, 3, 3) + Tanh & \multirow{2}{*}{\shortstack{FC(300) + Relu \\ Dropout(0.5)}}\\
                   & Conv(64, 5, 5) + Relu &  Conv(64, 8, 8) + Relu &  MaxPool(2,2) & \\
                   & Dropout(0.25) & Conv(128, 6, 6) + Relu  & Conv(64, 3, 3) + Tanh & \multirow{2}{*}{\shortstack{FC(300) + Relu \\ Dropout(0.5)}} \\
                   & FC(128) + Relu & Conv(128, 6, 6) + Relu & MaxPool(2,2) & \\
                   & Dropout(0.5) &  Dropout(0.5)  &  FC(128) + Relu &  \multirow{2}{*}{\shortstack{FC(300) + Relu \\ Dropout(0.5)}}\\
                   & FC + Softmax &  FC + Softmax &  FC + Softmax & \\
                   & & & & \multirow{2}{*}{\shortstack{FC(300) + Relu \\ Dropout(0.5)}} \\
                   & & & & \\
                   & & & & FC + Softmax \\
            \bottomrule
            \end{tabular}
            \caption{The different MNIST Architectures used for $f_s$ and $f_t$}
            \label{appendix:mnist_ens_adv_training_table}
    \end{table}

\paragraph{CIFAR and Imagenet model architectures} For table~\ref{table:non_robust_table1}, $f_s$ and $f_t$ are chosen randomly from an ensemble of trained classifiers. The ensemble is composed of five different architectures: VGG-16 \citep{simonyan2014very}, ResNet-18 (RN-18) \citep{he2016deep}, Wide ResNet (WR) \citep{zagoruyko2016wide}, DenseNet-121 (DN-121) \citep{huang2017densely} and Inception-V3 architectures (Inc-V3) \citep{szegedy2016rethinking}, with 5 trained models per architecture.

\subsection{Additional metrics}
In addition to the results provided in table~\ref{table:non_robust_table1}, we also provide two other metrics here: the stochastic competitive ratio in table~\ref{appendix:comp_ratio_non_robust} and the knapsack ratio table~\ref{appendix:knap_ratio_non_robust}. Where the knapsack ratio is defined as the sum value of $S_{\mathcal{A}}$ ---i.e. the sum of total loss, as selected by the online algorithm divided by the value of $S^*$ selected by the optimal offline algorithm.
We observe that the competitive ratio is not always a good metric to compare the actual performance of the different algorithms, since sometimes the online algorithm with the best competitive ratio is not the algorithm with the best fool rate. The knapsack ratio on the other hand seems to be a much better proxy for the actual performance of the algorithms, this is due to the fact that we're interested in picking elements that have have a good chance to fool the target classifier but are not necessarily the best possible attack.

\begin{table*}[ht]
\footnotesize
\caption{Competitive ratio on non-robust models using FGSM and PGD attacker and various online algorithms.}
\label{appendix:comp_ratio_non_robust}
 \begin{center}\begin{tabular}{ c c c c c c c c }
 \toprule
 & & \multicolumn{3}{c}{MNIST (competitive ratio)} & \multicolumn{3}{c}{CIFAR-10 (competitive ratio)}\\
 & Algorithm & $k=10$ & $k=100$ & $k=1000$ & $k=10$ & $k=100$ & $k=1000$ \\
 \midrule
 \multirow{6}{*}{\rotatebox[origin=c]{90}{FGSM}}
 & \textsc{Naive} & .006 $\pm$ .001 & .010 $\pm$ .000 & .098 $\pm$ .000 & .002 $\pm$ .000 & .010 $\pm$ .000 & .100 $\pm$ .000\\
 \cmidrule{2-8}
 & \textsc{Optimistic} & .063 $\pm$ .004 & .083 $\pm$ .003 & .197 $\pm$ .003 & .035 $\pm$ .002 & .064 $\pm$ .001 & .203 $\pm$ .001\\
 & \textsc{Virtual} & .048 $\pm$ .003 & .079 $\pm$ .003 & .201 $\pm$ .003 & .030 $\pm$ .002 & .073 $\pm$ .001 & .212 $\pm$ .001\\
 & \textsc{Single-Ref} & .070 $\pm$ .004 & .135 $\pm$ .006 & .181 $\pm$ .003 & .045 $\pm$ .002 & .109 $\pm$ .002 & .174 $\pm$ .001\\
 & \algoname & .072 $\pm$ .004 & .124 $\pm$ .005 & .270 $\pm$ .005 & .043 $\pm$ .002 & .107 $\pm$ .002 & .287 $\pm$ .002\\
 \midrule
 \multirow{6}{*}{\rotatebox[origin=c]{90}{PGD}}
 & \textsc{Naive} & .005 $\pm$ .001 & .010 $\pm$ .000 & .098 $\pm$ .000 & .001 $\pm$ .000 & .010 $\pm$ .000 & .100 $\pm$ .000\\
 \cmidrule{2-8}
 & \textsc{Optimistic} & .023 $\pm$ .002 & .036 $\pm$ .001 & .156 $\pm$ .001 & .033 $\pm$ .002 & .052 $\pm$ .002 & .157 $\pm$ .002\\
 & \textsc{Virtual} & .011 $\pm$ .001 & .049 $\pm$ .001 & .173 $\pm$ .001 & .028 $\pm$ .002 & .056 $\pm$ .002 & .160 $\pm$ .002\\
 &\textsc{Single-Ref} & .032 $\pm$ .002 & .067 $\pm$ .002 & .135 $\pm$ .001 & .042 $\pm$ .003 & .087 $\pm$ .003 & .145 $\pm$ .001\\
 & \algoname & .023 $\pm$ .002 & .059 $\pm$ .002 & .215 $\pm$ .002 & .040 $\pm$ .002 & .081 $\pm$ .003 & .200 $\pm$ .003\\
 \bottomrule
\end{tabular}\end{center} 
\end{table*}

\begin{table*}[ht]
\footnotesize
\caption{Knapsack ratio on non-robust models using FGSM and PGD attacker and various online algorithms.}
\label{appendix:knap_ratio_non_robust}
 \begin{center}\begin{tabular}{ c c c c c c c c }
 \toprule
 & & \multicolumn{3}{c}{MNIST (knapscak ratio in \%)} & \multicolumn{3}{c}{CIFAR-10 (knapsack ratio in \%)}\\
 & Algorithm & $k=10$ & $k=100$ & $k=1000$ & $k=10$ & $k=100$ & $k=1000$ \\
 \midrule
 \multirow{6}{*}{\rotatebox[origin=c]{90}{FGSM}}
 & \textsc{Naive} & 19.0 $\pm$ 0.3 & 19.5 $\pm$ 0.2 & 29.9 $\pm$ 0.2 & 16.8 $\pm$ 0.2 & 20.3 $\pm$ 0.1 & 28.7 $\pm$ 0.1\\
 \cmidrule{2-8}
 & \textsc{Optimistic} & 33.0 $\pm$ 0.6 & 33.1 $\pm$ 0.3 & 42.1 $\pm$ 0.3 & 32.7 $\pm$ 0.4 & 34.1 $\pm$ 0.2 & 42.8 $\pm$ 0.2\\
 & \textsc{Virtual} & 30.8 $\pm$ 0.5 & 34.2 $\pm$ 0.3 & 42.9 $\pm$ 0.3 & 32.9 $\pm$ 0.4 & 37.8 $\pm$ 0.2 & 45.0 $\pm$ 0.2\\
 & \textsc{Single-Ref} & 39.7 $\pm$ 0.6 & 41.5 $\pm$ 0.6 & 40.2 $\pm$ 0.3 & 37.5 $\pm$ 0.5 & 45.7 $\pm$ 0.4 & 37.9 $\pm$ 0.1\\
 & \algoname & 36.2 $\pm$ 0.6 & 41.1 $\pm$ 0.6 & 51.4 $\pm$ 0.5 & 39.4 $\pm$ 0.5 & 47.1 $\pm$ 0.4 & 55.5 $\pm$ 0.3\\
 \midrule
 \multirow{6}{*}{\rotatebox[origin=c]{90}{PGD}}
 & \textsc{Naive} & 27.2 $\pm$ 0.6 & 15.5 $\pm$ 0.3 & 25.9 $\pm$ 0.3 & 22.5 $\pm$ 0.3 & 26.8 $\pm$ 0.2 & 36.3 $\pm$ 0.2\\
 \cmidrule{2-8}
 & \textsc{Optimistic} & 37.2 $\pm$ 0.9 & 24.2 $\pm$ 0.5 & 35.8 $\pm$ 0.4 & 35.3 $\pm$ 0.6 & 35.6 $\pm$ 0.4 & 43.1 $\pm$ 0.4\\
 & \textsc{Virtual} & 35.6 $\pm$ 0.9 & 27.5 $\pm$ 0.6 & 38.8 $\pm$ 0.4 & 35.5 $\pm$ 0.6 & 37.2 $\pm$ 0.5 & 43.9 $\pm$ 0.4\\
 &\textsc{Single-Ref} & 46.9 $\pm$ 1.1 & 34.3 $\pm$ 0.8 & 32.3 $\pm$ 0.5 & 39.0 $\pm$ 0.7 & 42.6 $\pm$ 0.6 & 41.2 $\pm$ 0.3\\
 & \algoname & 41.5 $\pm$ 1.1 & 32.3 $\pm$ 0.8 & 46.5 $\pm$ 0.6 & 40.9 $\pm$ 0.7 & 42.9 $\pm$ 0.6 & 48.8 $\pm$ 0.5\\
 \bottomrule
\end{tabular}\end{center} 
\end{table*}

\begin{table*}[ht]
\footnotesize
\caption{Competitive ratio on robust models using FGSM and PGD attacker and various online algorithms.}
\label{appendix:comp_ratio_robust}
 \begin{center}\begin{tabular}{ c c c c c c c c }
 \toprule
 & & \multicolumn{3}{c}{MNIST (competitive ratio)} & \multicolumn{3}{c}{CIFAR-10 (comptetitive ratio)}\\
 & Algorithm & $k=10$ & $k=100$ & $k=1000$ & $k=10$ & $k=100$ & $k=1000$ \\
 \midrule
 \multirow{6}{*}{\rotatebox[origin=c]{90}{FGSM}}
 & \textsc{Naive} & 0.00 $\pm$ 0.00 & 0.01 $\pm$ 0.00 & 0.10 $\pm$ 0.00 & 0.00 $\pm$ 0.00 & 0.01 $\pm$ 0.00 & 0.10 $\pm$ 0.00\\
 \cmidrule{2-8}
 & \textsc{Optimistic} & 0.24 $\pm$ 0.00 & 0.17 $\pm$ 0.00 & 0.33 $\pm$ 0.00 & 0.05 $\pm$ 0.00 & 0.21 $\pm$ 0.00 & 0.33 $\pm$ 0.00\\
 & \textsc{Virtual} & 0.18 $\pm$ 0.00 & 0.17 $\pm$ 0.00 & 0.33 $\pm$ 0.00 & 0.09 $\pm$ 0.00 & 0.22 $\pm$ 0.00 & 0.33 $\pm$ 0.00\\
 & \textsc{Single-Ref} & 0.27 $\pm$ 0.00 & 0.31 $\pm$ 0.00 & 0.28 $\pm$ 0.00 & 0.07 $\pm$ 0.00 & 0.39 $\pm$ 0.00 & 0.28 $\pm$ 0.00\\
 & \algoname & 0.25 $\pm$ 0.00 & 0.27 $\pm$ 0.00 & 0.49 $\pm$ 0.00 & 0.11 $\pm$ 0.00 & 0.35 $\pm$ 0.00 & 0.49 $\pm$ 0.00\\
 \midrule
 \multirow{6}{*}{\rotatebox[origin=c]{90}{PGD}}
& \textsc{Naive}& 0.00 $\pm$ 0.00 & 0.01 $\pm$ 0.00 & 0.10 $\pm$ 0.00 & 0.00 $\pm$ 0.00 & 0.01 $\pm$ 0.00 & 0.10 $\pm$ 0.00\\
 \cmidrule{2-8}
 & \textsc{Optimistic} & 0.10 $\pm$ 0.00 & 0.13 $\pm$ 0.00 & 0.32 $\pm$ 0.00 & 0.01 $\pm$ 0.00 & 0.15 $\pm$ 0.00 & 0.31 $\pm$ 0.00\\
 & \textsc{Virtual} & 0.09 $\pm$ 0.00 & 0.14 $\pm$ 0.00 & 0.32 $\pm$ 0.00 & 0.02 $\pm$ 0.00 & 0.16 $\pm$ 0.00 & 0.32 $\pm$ 0.00\\
 & \textsc{Single-Ref} & 0.12 $\pm$ 0.00 & 0.23 $\pm$ 0.00 & 0.27 $\pm$ 0.00 & 0.01 $\pm$ 0.00 & 0.25 $\pm$ 0.00 & 0.27 $\pm$ 0.00\\
 & \algoname & 0.13 $\pm$ 0.00 & 0.21 $\pm$ 0.00 & 0.48 $\pm$ 0.00 & 0.02 $\pm$ 0.00 & 0.25 $\pm$ 0.00 & 0.47 $\pm$ 0.00\\
 \bottomrule
\end{tabular}\end{center} 

\end{table*}

\begin{table*}[ht]
\footnotesize
\caption{Knapsack ratio on robust models using FGSM and PGD attacker and various online algorithms.}
\label{appendix:knap_ratio_robust}
 \begin{center}\begin{tabular}{ c c c c c c c c }
 \toprule
 & & \multicolumn{3}{c}{MNIST (knapscak ratio in \%)} & \multicolumn{3}{c}{CIFAR-10 (knapsack ratio in \%)}\\
 & Algorithm & $k=10$ & $k=100$ & $k=1000$ & $k=10$ & $k=100$ & $k=1000$ \\
 \midrule
 \multirow{6}{*}{\rotatebox[origin=c]{90}{FGSM}}
  & \textsc{Naive} & 1.2 $\pm$ 0.1 & 2.2 $\pm$ 0.0 & 10.5 $\pm$ 0.1 & 9.9 $\pm$ 0.2 & 12.5 $\pm$ 0.1 & 19.7 $\pm$ 0.0\\
 \cmidrule{2-8}
 & \textsc{Optimistic} & 38.0 $\pm$ 0.5 & 26.5 $\pm$ 0.1 & 44.6 $\pm$ 0.1 & 48.8 $\pm$ 0.5 & 45.2 $\pm$ 0.1 & 48.3 $\pm$ 0.0\\
 & \textsc{Virtual} & 35.6 $\pm$ 0.4 & 27.0 $\pm$ 0.1 & 38.0 $\pm$ 0.1 & 50.9 $\pm$ 0.3 & 52.5 $\pm$ 0.1 & 50.0 $\pm$ 0.0\\
 &\textsc{Single-Ref} & 46.9 $\pm$ 0.5 & 45.2 $\pm$ 0.2 & 46.4 $\pm$ 0.1 & 59.7 $\pm$ 0.6 & 73.1 $\pm$ 0.3 & 41.3 $\pm$ 0.1\\
 & \algoname & 49.2 $\pm$ 0.4 & 41.2 $\pm$ 0.1 & 58.6 $\pm$ 0.1 & 66.2 $\pm$ 0.4 & 74.5 $\pm$ 0.1 & 70.5 $\pm$ 0.0\\
 \midrule
 \multirow{6}{*}{\rotatebox[origin=c]{90}{PGD}}
 & \textsc{Naive} & 1.3 $\pm$ 0.1 & 2.4 $\pm$ 0.0 & 10.7 $\pm$ 0.1 & 11.9 $\pm$ 0.6 & 14.6 $\pm$ 0.2 & 21.8 $\pm$ 0.1\\
 \cmidrule{2-8}
 & \textsc{Optimistic} & 31.1 $\pm$ 0.5 & 24.4 $\pm$ 0.1 & 42.7 $\pm$ 0.1 & 46.0 $\pm$ 1.4 & 45.3 $\pm$ 0.3 & 49.2 $\pm$ 0.1\\
 & \textsc{Virtual} & 29.9 $\pm$ 0.4 & 26.3 $\pm$ 0.1 & 37.9 $\pm$ 0.1 & 49.4 $\pm$ 1.2 & 52.4 $\pm$ 0.3 & 51.6 $\pm$ 0.1\\
 &\textsc{Single-Ref} & 39.5 $\pm$ 0.5 & 41.8 $\pm$ 0.2 & 43.3 $\pm$ 0.1 & 56.1 $\pm$ 2.1 & 69.5 $\pm$ 0.9 & 42.0 $\pm$ 0.4\\
 & \algoname & 41.3 $\pm$ 0.4 & 39.7 $\pm$ 0.1 & 57.9 $\pm$ 0.1 & 63.4 $\pm$ 1.2 & 72.7 $\pm$ 0.3 & 71.2 $\pm$ 0.1\\
 \bottomrule
\end{tabular}\end{center} 
\end{table*}

\clearpage
\subsection{Additional results}
\label{app:additional_results_small_k}
\paragraph{Same architecture} In addition to table~\ref{table:non_robust_table1} we also provide some results on MNIST where $f_s$ and $f_t$ always have the same architecture but have different weights. This is a slightly less challenging setting as shown in \citet{bose2020adversarial}, we also observe that in this setting the adversaries are very effective against the target model.

\begin{table*}[ht]
\footnotesize
\caption{Fool rate on non-robust models, where $f_s$ and $f_t$ have the same architecture, using FGSM and PGD attacker and various online algorithms.}
\label{appendix:comp_ratio_same_arch}
 \begin{center}\begin{tabular}{ c c c c c }
 \toprule
 & & \multicolumn{3}{c}{MNIST (Fool rate in \%)}\\
 & Algorithm & $k=10$ & $k=100$ & $k=1000$ \\
 \midrule
 \multirow{6}{*}{\rotatebox[origin=c]{90}{FGSM}}
  & \textsc{Naive} (lower bound) & 73.5 $\pm$ 0.5 & 72.3 $\pm$ 0.4 & 72.6 $\pm$ 0.4\\
  & \textsc{Opt} (Upper-bound) & 100.0 $\pm$ 0.0 & 99.7 $\pm$ 0.0 & 98.6 $\pm$ 0.1\\
 \cmidrule{2-5}
 & \textsc{Optimistic} & 89.8 $\pm$ 0.4 & 86.0 $\pm$ 0.2 & 84.9 $\pm$ 0.2\\
 & \textsc{Virtual} & 90.3 $\pm$ 0.3 & 90.0 $\pm$ 0.2 & 88.1 $\pm$ 0.2\\
 &\textsc{Single-Ref} & 94.0 $\pm$ 0.3 & 96.3 $\pm$ 0.2 & 79.3 $\pm$ 0.3\\
 & \algoname & 96.9 $\pm$ 0.2 & 98.6 $\pm$ 0.1 & 97.5 $\pm$ 0.1\\
 \midrule
 \multirow{6}{*}{\rotatebox[origin=c]{90}{PGD}}
 & \textsc{Naive} (lower bound) & 91.1 $\pm$ 0.5 & 90.2 $\pm$ 0.4 & 90.0 $\pm$ 0.3\\
 & \textsc{Opt} (Upper-bound) & 98.5 $\pm$ 0.2 & 98.0 $\pm$ 0.1 & 97.4 $\pm$ 0.1\\
 \cmidrule{2-5}
 & \textsc{Optimistic} & 95.3 $\pm$ 0.3 & 93.8 $\pm$ 0.2 & 93.5 $\pm$ 0.2\\
 & \textsc{Virtual} & 95.5 $\pm$ 0.3 & 95.2 $\pm$ 0.2 & 94.4 $\pm$ 0.2\\
 &\textsc{Single-Ref} & 96.7 $\pm$ 0.3 & 96.9 $\pm$ 0.2 & 92.0 $\pm$ 0.3\\
 & \algoname & 97.1 $\pm$ 0.3 & 97.6 $\pm$ 0.1 & 97.0 $\pm$ 0.1\\
 \bottomrule
\end{tabular}\end{center} 
\end{table*}

\begin{table*}[ht]
\footnotesize
\caption{Competitive ratio on non-robust models for $k=4$}
\label{appendix:fool_rate_k=4}
 \begin{center}\begin{tabular}{ c c c c}
 \toprule
 & & MNIST (competitive ratio) & CIFAR (competitive ratio)\\
 & Algorithm & $k=4$ & $k=4$\\
 \midrule
 \multirow{6}{*}{\rotatebox[origin=c]{90}{FGSM}}
  & \textsc{Naive} & 0.004 $\pm$ 0.002 & 0.002 $\pm$ 0.001\\
 \cmidrule{2-4}
 & \textsc{Optimistic} & 0.194 $\pm$ 0.016 & 0.152 $\pm$ 0.012\\
 & \textsc{Virtual} & 0.147 $\pm$ 0.013 & 0.121 $\pm$ 0.011\\
 &\textsc{Single-Ref} & 0.200 $\pm$ 0.016 & 0.160 $\pm$ 0.013\\
 & \algoname & 0.199 $\pm$ 0.016 & 0.147 $\pm$ 0.012\\
 \midrule
 \multirow{6}{*}{\rotatebox[origin=c]{90}{PGD}}
 & \textsc{Naive} & 0.001 $\pm$ 0.001 & 0.000 $\pm$ 0.000\\
 \cmidrule{2-4}
 & \textsc{Optimistic} & 0.119 $\pm$ 0.013 & 0.075 $\pm$ 0.021\\
 & \textsc{Virtual} & 0.089 $\pm$ 0.010 & 0.070 $\pm$ 0.019\\
 &\textsc{Single-Ref} & 0.119 $\pm$ 0.013 & 0.059 $\pm$ 0.019\\
 & \algoname & 0.132 $\pm$ 0.013 & 0.086 $\pm$ 0.022\\
 \bottomrule
\end{tabular}\end{center} 
\end{table*}

\begin{table*}[ht]
\footnotesize
\caption{Knapsack ratio on non-robust models for $k=4$}
\label{appendix:knapsack_ratio_k=4}
 \begin{center}\begin{tabular}{ c c c c}
 \toprule
 & & MNIST (Knapsack ratio in \%) & CIFAR (Knapsack ratio in \%)\\
 & Algorithm & $k=4$ & $k=4$\\
 \midrule
 \multirow{6}{*}{\rotatebox[origin=c]{90}{FGSM}}
  & \textsc{Naive} & 12.5 $\pm$ 0.3 & 15.6 $\pm$ 0.3\\
 \cmidrule{2-4}
 & \textsc{Optimistic} & 34.5 $\pm$ 0.7 & 36.7 $\pm$ 0.6\\
 & \textsc{Virtual} & 31.4 $\pm$ 0.6 & 33.2 $\pm$ 0.5\\
 &\textsc{Single-Ref} & 34.9 $\pm$ 0.7 & 37.4 $\pm$ 0.6\\
 & \algoname & 37.2 $\pm$ 0.7 & 38.9 $\pm$ 0.6\\
 \midrule
 \multirow{6}{*}{\rotatebox[origin=c]{90}{PGD}}
 & \textsc{Naive} & 21.0 $\pm$ 0.6 & 71.3 $\pm$ 1.6\\
 \cmidrule{2-4}
 & \textsc{Optimistic} & 39.6 $\pm$ 1.0 & 82.8 $\pm$ 1.6\\
 & \textsc{Virtual} & 35.5 $\pm$ 0.9 & 81.8 $\pm$ 1.6\\
 &\textsc{Single-Ref} & 40.8 $\pm$ 1.0 & 81.7 $\pm$ 1.5\\
 & \algoname & 42.7 $\pm$ 1.0 & 84.4 $\pm$ 1.5\\
 \bottomrule
\end{tabular}\end{center} 
\end{table*}

\begin{table*}[ht]
\footnotesize
\caption{Fool rate on non-robust models for $k=4$}
\label{appendix:fool_rate_k=4}
 \begin{center}\begin{tabular}{ c c c c}
 \toprule
 & & MNIST (Fool rate in \%) & CIFAR (Fool rate in \%)\\
 & Algorithm & $k=4$ & $k=4$\\
 \midrule
 \multirow{6}{*}{\rotatebox[origin=c]{90}{FGSM}}
  & \textsc{Naive} (lower bound) & 59.2 $\pm$ 0.9 & 59.6 $\pm$ 0.8\\
  & \textsc{Opt} (upper bound) & 92.6 $\pm$ 0.5 & 86.6 $\pm$ 0.6 \\
 \cmidrule{2-4}
 & \textsc{Optimistic} & 83.7 $\pm$ 0.7 & 79.8 $\pm$ 0.7\\
 & \textsc{Virtual} & 80.6 $\pm$ 0.7 & 76.4 $\pm$ 0.7\\
 &\textsc{Single-Ref} & 84.9 $\pm$ 0.7 & 80.5 $\pm$ 0.7\\
 & \algoname & 86.5 $\pm$ 0.6 & 82.7 $\pm$ 0.7\\
 \midrule
 \multirow{6}{*}{\rotatebox[origin=c]{90}{PGD}}
 & \textsc{Naive} (lower bound) & 59.9 $\pm$ 1.2 & 71.3 $\pm$ 1.6 \\
 & \textsc{Opt} (Upper-bound) & 79.3 $\pm$ 1.0 & 87.7 $\pm$ 1.3 \\
 \cmidrule{2-4}
 & \textsc{Optimistic} & 72.1 $\pm$ 1.1 & 82.8 $\pm$ 1.6\\
 & \textsc{Virtual} & 70.0 $\pm$ 1.1 & 81.8 $\pm$ 1.6\\
 &\textsc{Single-Ref} & 74.1 $\pm$ 1.1 & 81.7 $\pm$ 1.5\\
 & \algoname & 74.5 $\pm$ 1.1 & 84.4 $\pm$ 1.5\\
 \bottomrule
\end{tabular}\end{center} 
\end{table*}

\clearpage

\section{Distribution of Values Observed By Online Algorithms}
\label{appendix:visualization_of_values_observed} 
In this section we further investigate performance disparity of online algorithms against robust and non-robust models for CIFAR-10 as observed in Tables \ref{table:non_robust_table1} and \ref{table:madry_challenge}. We hypothesize that one possible explanation can be found through analyzing the ratio distribution of values $\mathcal{V}_i$'s for unsuccessful and successful attacks as observed by the online algorithm when attacking each model type. However, note that eventhough an online adversary may employ a fixed attack strategy to craft an attack $x'=\textsc{ATT}(x)$ the scale of values in each setting are not strictly comparable as the attack is performed on different model types.
In other words, given an $\textsc{ATT}$ it is significantly more difficult to attack a robust model and thus we can expect a lower $\mathcal{V}_i$ when compared to attacking a non-robust model. Thus to investigate the difference in efficacy of online attacks we pursue a distributional argument. %

Indeed, distributions of $\mathcal{V}_i$'s observed, for a specific permutation of $\mathcal{D}$, may drastically affect the performance of the online algorithms. Consider for instance, if the $\mathcal{V}_i$'s that correspond to successful attacks cannot be distinguished from the ones that are unsuccessful. In such a case one cannot hope to use an online algorithm---that only observes $\mathcal{V}_i$'s---to always correctly pick successful attacks. In Figure \ref{fig:distrib_values} we visualize the ratio of $\mathcal{V}_i$'s of unsuccessful and successful attacks as the ratio of the densities of unsuccessful versus successful attack vectors (y-axis) as provided by a kernel density estimator for CIFAR-10 robust and non-robust models. It provides a non-normalized value of the ratio of unsuccessful attacks for a given value of $\mathcal V_i$. 
As observed, there is a significant amount of non-successful attacks for large values of $\mathcal V_i$ in the non-robust case which indicates that there are many data points with high values that lead to unsuccessful attacks. Furthermore, this also suggests one explanation for the higher efficacy of online algorithms against robust models: fewer attacks are successful but they are easier to differentiate from unsuccessful ones because of their relatively larger loss value. Importantly, this implies that given an online attack budget $k \ll n$ higher online fool rates can be achieved against robust models as the selected data points turn adversarial with higher probability when compared to non-robust models.


\begin{figure}[H]
\centering
    \includegraphics[width=\textwidth]{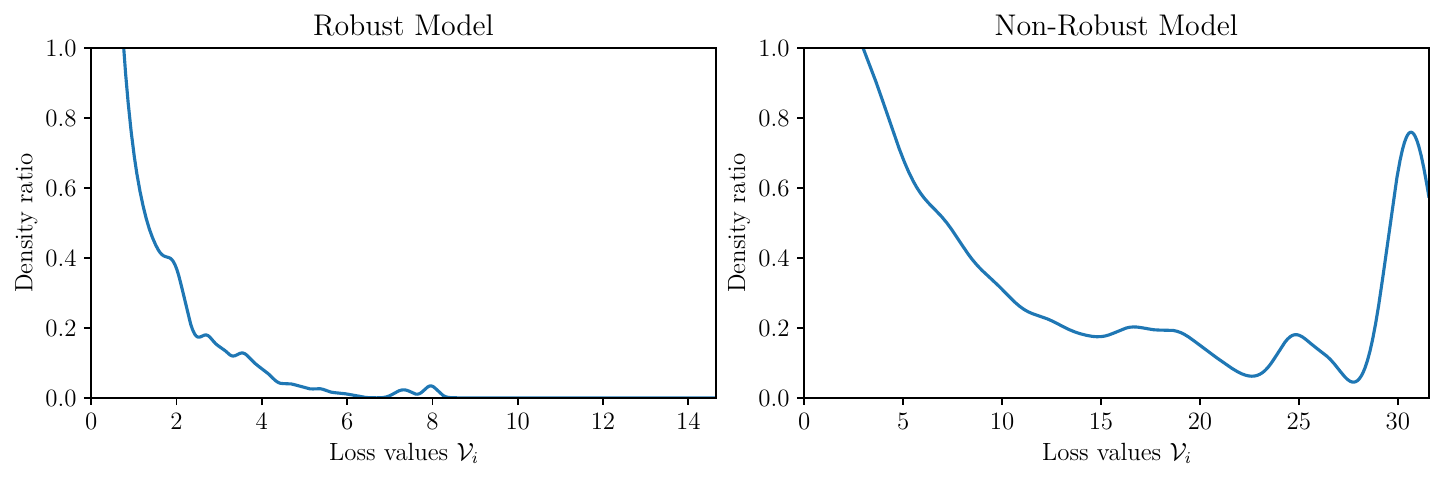}
    \caption{Distribution of the values for robust and non-robust models. We use a gaussian kernel density estimator to estimate the density.}
\label{fig:distrib_values}
\end{figure}

\section{Related Work}

\xhdr{Adversarial attacks} The idea of attacking deep networks was first introduced in \citep{szegedy2013intriguing,goodfellow2014explaining}, and recent years have witnessed the introduction of challenging threat models, such as blackbox \cite{chen2017zoo,ilyas2018black, jiang2019black,bose2020adversarial,chakraborty2018adversarial} as well as defense strategies \cite{madry2017towards,tramer2017ensemble,Ding2020MMA}.
Closest to our setting are adversarial attacks against real-time systems \cite{gong2019real, gong2019remasc} and deep reinforcement learning agents \cite{lin2017tactics,sun2020stealthy}. However, unlike our work, these are not online algorithms and do not impose online constraints (see~\S\ref{appendix:further_related_work}).

\xhdr{$k$-secretary} The classical secretary problem was originally proposed by \cite{gardner1960mathematical} and later solved in \cite{dynkin1963optimum} with an $(1/e)$-optimal algorithm. \citet{kleinberg2005multiple} introduced the $k$-secretary problem and an asymptotically optimal algorithm achieving a competitive ratio of $1 - \Theta(\sqrt{1/k})$. 
As outlined in \S\ref{connection_to_prior_work} for general $k$ an optimal algorithms exist \cite{chan2014revealing}, but requires the analysis of involved LPs that grow with the size of $n$. A parallel line of work dubbed the prophet secretary problem, considers online problems where---unlike \S\ref{stochastic_secretary_algorithms_section}---some information on the distribution of values is known \textit{a priori}~\cite{azar2014prophet, azar2018prophet, esfandiari2017prophet}. Secretary problems have also been applied to machine learning by informing the online algorithm about the inputs before execution \cite{antoniadis2020secretary,dutting2020secretaries}. 
Finally, other interesting secretary settings include playing with adversaries \cite{bradac2019robust, kaplan2020competitive}.

\label{appendix:further_related_work}
 While we consider ---to the best of our knowledge---that our work is the only truly online threat model. Our setting is the only one considering that data points are only ever observed once, and a decision to attack must be made at the moment and cannot be reversed retroactively. For example, \citep{gong2019remasc} consider replay attacks on Voice-Controlled Systems whereby streamed audio input is captured with a recording device, and then the entire sequence is spoofed and replayed back. Unlike online attacks that we consider, they can manipulate the whole sequence retroactively and do not have to make an irreversible decision to attack at a given timestep. Similarly, both \citep{lin2017tactics, sun2020stealthy} consider adversarial attacks against deep reinforcement learning agents. Like us, they consider an adversarial budget that limits the number of points to attack to avoid detection. However, unlike us, they require whitebox access to the target model in order to train another predictive model by interacting with the environment, which can later be used to inform “when to attack”. Thus the datapoints appearing at test time may already be seen and scored during the training period. The dichotomy between collecting data for potentially an infinite time horizon before attacking is at odds with our online threat model as a data point can only be observed once. As a result, none of these works can be used within our online threat model and are not appropriate baselines.
\end{document}